\documentclass[preprint,12pt]{elsarticle}
\usepackage{hyperref}
\usepackage{times}
\usepackage{amscd,amsmath,amssymb,amsthm}
\usepackage{color} 

\usepackage{mathrsfs}
\usepackage{graphicx}

\usepackage{helvet}
\usepackage{courier}

\usepackage{algorithm}
\usepackage[noend]{algorithmic}
\usepackage{verbatim}
\usepackage{flafter}

\newcommand{\lev}{\textsf{lev}}
\newcommand{\atom}{\textsc{Atom}}
\newcommand{\face}{\textsc{Face}}
\newcommand{\hole}{\textsc{Hole}}

\newcommand{\dc}{\textsc{DC}}
\newcommand{\ec}{\textsc{EC}}
\newcommand{\po}{\textsc{PO}}
\newcommand{\tpp}{\textsc{TPP}}
\newcommand{\ntpp}{\textsc{NTPP}}
\newcommand{\eq}{\textsc{EQ}}
\newcommand{\tppi}{\textsc{TPPi}}
\newcommand{\ntppi}{\textsc{NTPPi}}

\theoremstyle{definition}
\newtheorem{definition}{Definition}
\theoremstyle{plain}
\newtheorem{proposition}{Proposition}
\newtheorem*{prop-other}{Proposition}
\newtheorem{theorem}{Theorem}
\newtheorem*{thm-other}{Theorem}
\newtheorem{example}{Example}
\newtheorem{corollary}{Corollary}
\newtheorem{lemma}{Lemma}

\theoremstyle{remark}
\newtheorem{remark}{Remark}

\begin{document}

\journal{arXiv}

\begin{frontmatter}

\title{On the Internal Topological Structure of Plane Regions\tnoteref{aaai}} \tnotetext[aaai]{An abstract of this paper appeared  in \emph{Proceedings of the 12th International Conference on  the Principles of Knowledge Representation and Reasoning (KR-10)}, pages 581-583, Toronto, Canada, May 9-13, 2010.}

\author[1]{Sanjiang Li\corref{cor1}}
\ead{sanjiang.li@uts.edu.au}

\address[1]{Centre for Quantum Computation and Intelligent Systems,
       Faculty of Engineering and Information Technology, University of Technology
       Sydney, Australia}

\cortext[cor1]{Corresponding Author}
\begin{abstract}
The study of topological information of spatial objects has for a long time been a focus of research in disciplines like computational geometry, spatial reasoning, cognitive science, and robotics. While the majority of these researches emphasised the topological relations between spatial objects, this work studies the internal topological structure of bounded plane regions, which could consist of multiple pieces and/or have holes and islands to any finite level. The insufficiency of simple regions (regions homeomorphic to closed disks) to cope with the variety and complexity of spatial entities and phenomena has been widely acknowledged. Another significant drawback of simple regions is that they are not closed under set operations union, intersection, and difference.
This paper considers bounded semi-algebraic regions, which are closed under set operations and can closely approximate most plane regions arising in practice. 

For each bounded semi-algebraic region $A$, we associate a unique set of \emph{faces} and a unique set of \emph{holes}, and show that each face/hole of $A$ is a \emph{simple region with holes} and the union of all faces and holes (called the \emph{envelope} of $A$) is a \emph{composite region}. We further define the \emph{atoms} of $A$ as those simple regions involved in the envelope, the faces, and the holes of $A$, and prove that these atoms are necessary and sufficient to determine the (global) nine-intersection (topological) relations between $A$ and other regions. The internal topological structure of $A$ is then represented in a layered graph model (called the \emph{link graph} of $A$), where the nodes represent the connected components of the interior and exterior of $A$, and two nodes are linked if their boundaries share an arc. Compared with the tree model of Worboys and Bofakos, the link graph is more precise and can answer queries such as ``\emph{How many faces does $A$ (or a hole of $A$) have?}". Moreover, the tree model of $A$ can be efficiently computed from the the link graph of $A$. 

\end{abstract}


\end{frontmatter}

\section{Introduction}
Qualitative spatial reasoning (QSR) is an established research subfield in Artificial Intelligence and Geographical Information Science (GISc). One basic requirement of QSR is to provide qualitative, non-numerical information at various levels of detail for spatial representation and reasoning. QSR is an interdisciplinary research of computer science, cognition, and geography, with important applications in a number of areas including Geographical Information Systems (GISs), Content-Based Image Retrieval, Computer Graphics, Robotics \citep{CohnR07}.

Among the many aspects of space, topology is the most basic and important. A major part of QSR research focuses on the study of topological
relations  and topological properties. The 9-Intersection Model (9IM) \citep{EgenhoferH90}, perhaps the most well-known
topological relation model in GISc, identifies the same set of eight basic topological relations as the popular RCC8 model \citep{RandellCC92} in qualitative spatial reasoning. The 9IM was initially defined for simple regions \citep{EgenhoferH90}, \emph{i.e.} regions that are homeomorphic to a closed disk (see Figure~\ref{fig:RCC8}). The insufficiency of simple regions for representing spatial phenomena and entities, \emph{e.g.} countries, has been widely acknowledged. Various models have therefore been proposed to represent complicated spatial regions. These include \emph{simple regions with holes} \citep{EgenhoferCF94}, \emph{composite regions} \citep{ClementiniDC95}, and \emph{complex regions} \citep{SchneiderB06} (cf. Figure~\ref{fig:3comb}). Moreover, the 9IM has been extended to represent topological relations between complex regions \citep{Li06-ijgis,SchneiderB06}.
\begin{figure}[h]
\centering
\begin{tabular}{c}
\includegraphics[width=.95\textwidth]{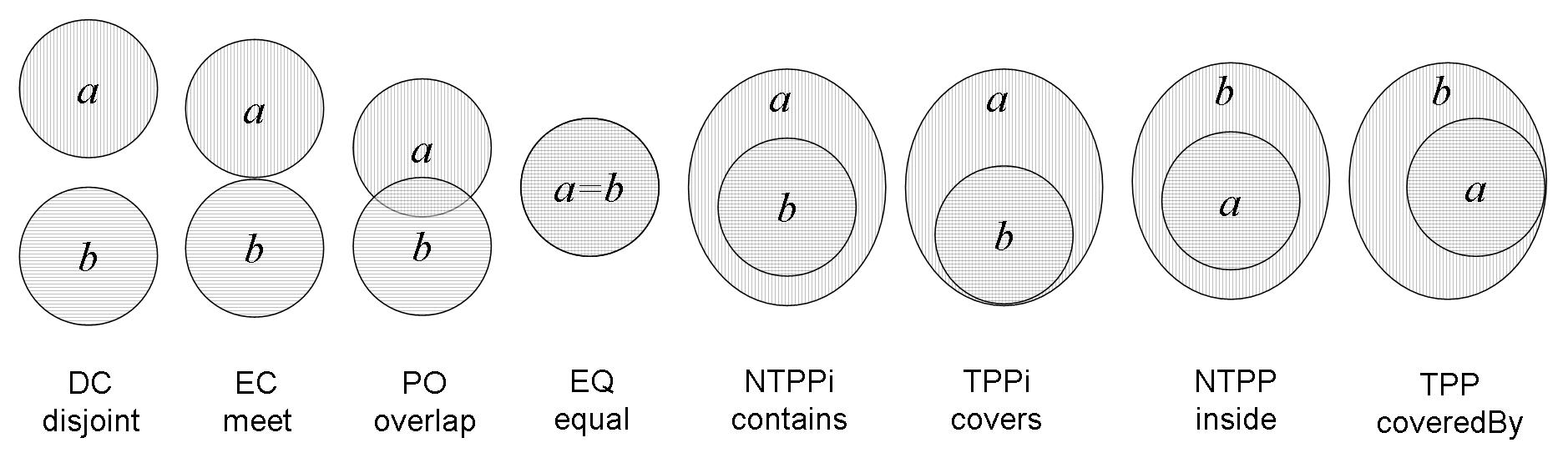}
\end{tabular}
\caption{The eight topological relations between two simple regions}
\label{fig:RCC8}
\end{figure}


This paper will not propose a new topological relation model. Instead, we will focus on the \emph{internal topological structure} of complex regions, \emph{i.e.} how a complex region is  constructed from a particular set of atomic simple regions. A better understanding of the internal topological structure of a complex region may help us (i) in answering queries such as ``Does $X$ have a hole that has an island?" and ``Does $X$ have a connected component that is disjoint from all other connected components of $X$?", (ii) in enhancing spatial reasoning by, \emph{e.g.} reducing the ambiguity of composition-based qualitative inferences \citep{VasardaniE08}; and (iii) in describing dynamic changes of spatial regions \citep{JiangW09}.

The aim of this paper is to establish a qualitative model for characterizing the internal topological structure of complex regions. For each bounded region $X$, we try to decompose $X$ into a set of \emph{connected components} and then represent the internal topological structure of $X$ using these components. We call  the closure of each connected component of $X^\circ$ (the interior of $X$) a \emph{face} of $X$, and call the closure of each bounded connected component of $X^e$ (the exterior of $X$) a \emph{hole} of $X$ (see Section~2.1 for unexplained notions).  We call the union of all faces and holes of $X$ the \emph{envelope} of $X$, written as $\widehat{X}$.

On the basis of certain practical assumptions, we are able to show that each face/hole of a bounded region $X$ is a \emph{simple region with holes} \citep{EgenhoferCF94} and the envelope of $X$ is a \emph{composite region} \citep{ClementiniDC95}. This suggests that a complex region can be represented as a structured combination of simple regions with holes and composite regions.  
Inspired by this observation, we construct a layered graph for each bounded region $X$, which characterizes most of the internal topological structure of $X$. The graph has a root node, which represents the unbounded connected component of the exterior of $X$, and each non-root node represents a bounded connected component of either $X^\circ$ or $X^e$, and two nodes are connected if they share a common arc, which happens only when exactly one is a component of the interior of $X$. We call this graph the \emph{link graph} of $X$, which reflects most of the internal topological structure of $X$. In particular, from the link graph of $X$, we can always answer basic queries such as ``\emph{How many faces does $X$ (or a hole of $X$) have}", which the tree model of Worboys and Bofakos \cite{WorboysB93} cannot always answer (cf. Section~6.2 of this paper). Moreover, the tree model of $X$ can also be efficiently computed from the link graph of $X$, but not vice versa.

We further define the \emph{atoms} of $X$ as those simple regions involved in the envelope, the faces, and the holes of $X$. Consider the regions illustrated in Figure~\ref{fig:3comb}. It is evident that $A$ has one face, one hole, and two atoms; $B$ has two faces, and two atoms; and $C$ has two faces, one hole, and four atoms.


\begin{figure}[h]
\centering
\begin{tabular}{ccc}
\includegraphics[width=.22\textwidth]{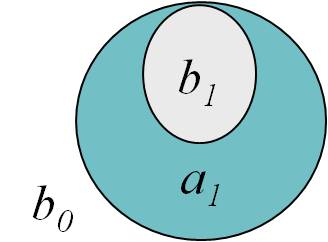}
&
\includegraphics[width=.2\textwidth]{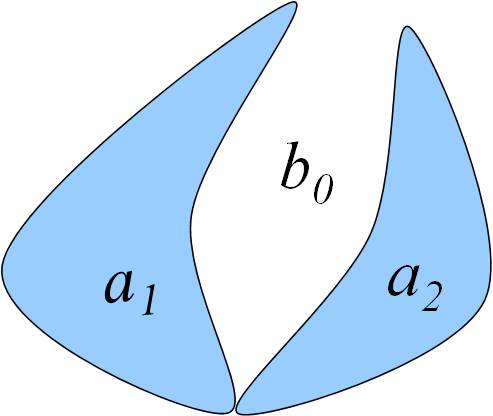}
&
\includegraphics[width=.22\textwidth]{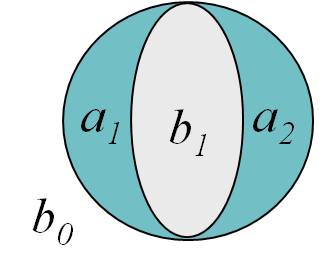}
\\
$A$ & $B$ & $C$
\end{tabular}
\caption{Three plane regions $A,B,C$ 
}
\label{fig:3comb}
\end{figure}

One very interesting property of this atom set is that it is necessary and sufficient for \emph{locally} determining the \emph{global} 9-intersection relation  \citep{Li06-ijgis,SchneiderB06} between bounded regions. This means that, on one hand, the 9-intersection relation between two bounded regions $A,B$ is uniquely determined by the topological RCC8 relations between atoms of $A,B$; on the other hand, the 9-intersection relation between $A$ and some $B'$ may be undetermined if any simple region is removed from the atom set of $A$.

The rest of this paper proceeds as follows. Section 2 introduces basic topological notions and various definitions of spatial regions, as well as some preliminary results. Note that plane regions may have arbitrarily complicated topological properties. For practical applications, it is necessary to consider regions that are finitely representable. Section~3 introduces the definition of semi-algebraic regions and proves that a face/hole (the envelope, resp.) of a bounded semi-algebraic region is a simple region with holes (composite region, resp.). Section~4 establishes a complete characterization of the (\emph{global}) 9-intersection relations between two bounded semi-algebraic regions \emph{locally} in terms of the RCC8 relations between their atoms. The graph model is then introduced and discussed in Section~5. A detailed comparison of our graph model with related works, in particular the tree model of Worboys and Bofakos \cite{WorboysB93}, is given in Section~6. The last section concludes the paper.

\section{Backgrounds}
In this section, we first introduce basic topological notions in general, and the reader is referred to \citep{Munkres00} for a detailed introduction of point-set topology. We then concern ourselves with the real plane (with the usual metric and topology), and introduce several important concepts related to bounded plane regions. In particular, we will define a face or hole of a bounded region, and define when a bounded region is a \emph{simple region with holes} or a \emph{composite region}.
\subsection{Basic Topological Notions}

A \emph{topology} $\mathcal{T}$ over a nonempty set $X$ is a subset of the powerset $\wp(X)$ that is closed under arbitrary unions and finite intersections. As a consequence, we know each topology over $X$ contains the empty set $\varnothing$ and $X$ itself.  We call $(X,\mathcal{T})$ a \emph{topological space}, and call each set in $\mathcal{T}$ an \emph{open set}. For a subset $A$ of $X$, the \emph{interior} of $A$, denoted by $A^\circ$, is the largest open set which is contained in $A$; the \emph{exterior} of $A$, denoted by $A^e$, is the interior of $X\setminus A$, the complement of $A$  (cf. Figure~\ref{fig:pregion}(a)). Given a topological space $(X,\mathcal{T})$, a set $A\subseteq X$ is
a \emph{closed set} if its complement is open. For
any $B\subseteq X$, the \emph{closure} of $B$, written $\overline{B}$, is the smallest closed
set which contains $B$. A closed set $A$ is
\emph{regular} if $\overline{A^\circ}=A$. The \emph{boundary} of a subset $A$ of $X$, denoted by $\partial A$,
is defined to be the set difference of $\overline{A}$ and $A^\circ$,
\emph{i.e.} $\partial A=\overline{A}\setminus A^\circ$ (cf. Figure~\ref{fig:pregion}(a)).

A subset $A$ of a topological space $(X,\mathcal{T})$ is \emph{connected} if for any two disjoint open sets $U,V$ such
that $A\subseteq U\cup V$ we have either $U\cap A$ or $V\cap A$  is empty. A connected set $U$ is called a \emph{connected component} of $A$ if $U$ is a maximally connected subset of $A$. Two different connected components are clearly disjoint.

A map $f:X\rightarrow Y$ between two topological spaces is \emph{continuous} iff, for every open set $U\subset Y$, the inverse image $f^{-1}(U)=\{x\in X: f(x)\in U\}$ is open in $X$. A bijective map $f$ between two topological spaces is a \emph{homeomorphism} if both $f$ and its inverse map $f^{-1}$ are continuous. A \emph{path} from a point $P$ to a point $Q$ in $X$ is a continuous map $f$ from the unit interval $[0,1]$ (with the usual topology) to $X$ with $f(0)=P$ and $f(1)=Q$. A subset $U$ of $X$ is \emph{path-connected} if there is a path joining any two points in $U$.

\subsection{Complex Regions and Their Components}
In this paper, we consider one particular topological space --- the real plane $\mathbb{R}^2$, with the topology induced by the usual metric. Let $P,Q$ be two points in $\mathbb{R}^2$, and $X$ be a subset of $\mathbb{R}^2$. We write $d(P,Q)$ for the distance between $P,Q$, and write $d(P,X)=\inf\{d(P,Q): Q\in X\}$. For $P\in\mathbb{R}^2$, $\delta>0$, we write $B^\circ(P,\delta)=\{Q: d(P,Q)<\delta\}$ for the open disk centred at $P$ with radius $\delta$.

For a set $A$ in $\mathbb{R}^2$, we say $A$ is \emph{bounded} if $A$ is contained in a disk, $A$ is \emph{open} if $A$ contains an open disk centred at $P$ for each $P$ in $A$, we say $A$ is a \emph{plane region} is $A$ is a regular closed set in $\mathbb{R}^2$. Figure~\ref{fig:pregion} shows a bounded plane region, and its interior, exterior, and boundary.

A \emph{Jordan arc} in $\mathbb{R}^2$ is the image of an injective continuous map of a closed interval into the plane, and a \emph{Jordan curve} or a \emph{simple closed curve} in $\mathbb{R}^2$ is the image of a continuous map $\varphi:[0,1]\rightarrow \mathbb{R}^2$ such that $\varphi(0)=\varphi(1)$ and the restriction of $\varphi$ to $[0,1)$ is injective. Let $C$ be a Jordan curve in the plane. Jordan Curve Theorem says that the complement of $C$ in the plane consists of exactly two connected components. One is bounded, the other is unbounded, and the curve $C$ is the boundary of each component. 

It is well-known that any connected open subset of $\mathbb{R}^2$ is path-connected.

\begin{figure}[h]
\centering
\begin{tabular}{cc}
\includegraphics[width=.55\textwidth]{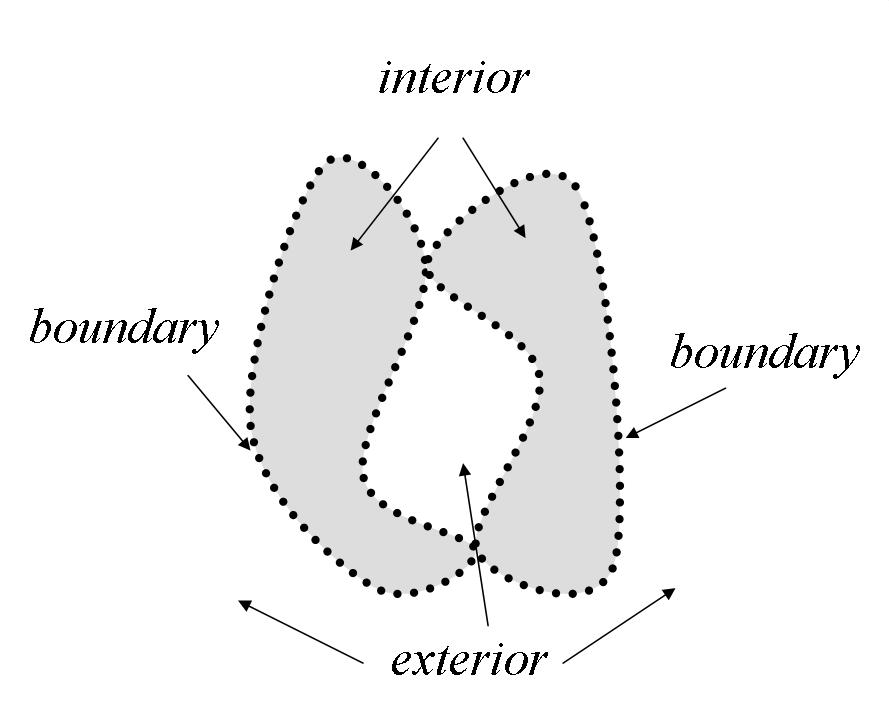}
&
\includegraphics[width=.26\textwidth]{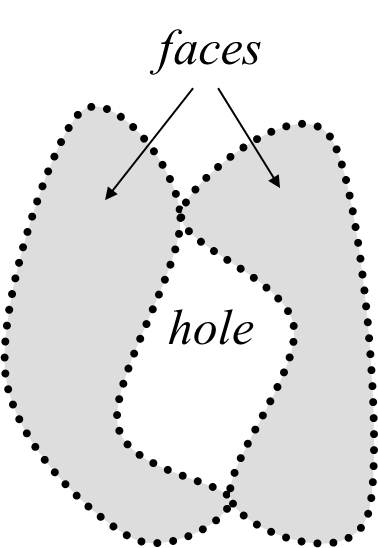}\\
(a) & (b)
\end{tabular}
\caption{A complex region and its (a) interior, boundary, and exterior; (b) faces and hole}
\label{fig:pregion}
\end{figure}
Because a plane region is a closed set, its set complement is open. We have 
\begin{proposition}\label{prop:complement=b0}
Let $A$ be a plane region. Then the exterior of $A$ is exactly its set complement. Assume, moreover, $A$ is bounded. Then  the exterior of $A$ has a unique unbounded connected component.
\end{proposition}
\begin{definition}\label{dfn:components}
For a bounded plane region $A$, we call a connected component of the interior (exterior, resp.) of $A$ an \emph{interior} (\emph{exterior}, resp.) component of $A$, and call the \underline{closure} of each interior (bounded exterior, resp.) component of $A$ a \emph{face} (\emph{hole}, resp.) of $A$.  We  write $\face(A)$ ($\hole(A)$, resp.) for the set of faces (holes, resp.) of $A$, and write $\widehat{A}$ for the union of all faces and holes of $A$, called the \emph{envelope} of $A$.
\end{definition}
For example, the region shown in Figure~\ref{fig:pregion} has two faces and one hole.  We note that a component is always an open connected set.


It is straightforward to show that each face/hole is a regular closed set, \emph{i.e.}

\begin{proposition}
The closure of each interior or exterior component of a bounded plane region is a plane region.
\end{proposition}
This explains why we consider in Definition~\ref{dfn:components} the closure of connected components of $A^\circ$ and $A^e$. A bounded region may have an infinite number of faces and/or an infinite number of holes, but disk-like regions have only one face and no hole. Topologically speaking, these are the simplest plane regions.
\begin{definition}  \label{dfn:simpleregion}
A set $A$ on the real plane is called a \emph{simple region} if it is homeomorphic to a closed disk.
\end{definition}
We next recall two other special kinds of bounded plane regions. We note that the definitions given here are slightly different from the original definitions given in \citep{EgenhoferCF94} and \citep{ClementiniDC95}. Comparisons will be given later in Section~6.
\begin{definition}\citep{EgenhoferCF94} \label{dfn:srh}
A plane region $A$ is called a \emph{simple region with holes} if there exist
a set of simple regions $\{a_0,a_1,\cdots, a_k\}$ ($k\geq 0$) such that
\begin{itemize}
\item [(1)] $A=a_0\setminus\bigcup_{i=1}^k a_i^\circ$, and $A$ has a connected interior;
\item [(2)] $a_i\subset a_0$, and $\partial a_i\cap\partial a_0$ is either empty or a singleton for each
 $1\leq i\leq k$;
\item [(3)] each $a_i\cap a_j$ is either empty or a singleton for any $1\leq i\not=j \leq k$.
\end{itemize}
For clarity, we often write $A=(a_0;a_1,\cdots,a_k)$. 
\end{definition}
For a simple region with holes $A=(a_0;a_1,\cdots,a_k)$, it is worth noting that $a_1,\cdots,a_k$ are the only holes of $A$ in the sense of Definition~\ref{dfn:components}, and $A$ is its only face. Also note that $a_0$ is the union of the face and holes of $A$, hence identical to $\widehat{A}$, the \emph{envelope} of $A$. In \citep{EgenhoferCF94}, $a_0$ is called the \emph{generalized region} of $A$.

A simple region with holes has a connected interior and, possibly, a disconnected exterior. By contrast, a \emph{composite region} has a connected exterior and, possibly, a disconnected interior.

\begin{definition}\citep{ClementiniDC95}
\label{dfn:composite-region}
A plane region $A$ is called a
\emph{composite region} if there exist a set of simple regions
$\{a_1,a_2,\cdots,a_k\}$ ($k\geq 1$) such that
\begin{itemize}
\item [(1)]
$A=\bigcup_{i=1}^k a_i$ and $A$ has a connected exterior;
\item [(2)] each $a_i\cap a_j$ is either empty or a singleton for any $1\leq i\not=j\leq n$.
\end{itemize}
\end{definition}
It is clear that the composite region in the above definition contains $k$ faces (\emph{i.e.} $a_1,\cdots,a_k$) and has no holes.
We stress that in this paper when we say a region is a composite region, it is always a bounded region as in the above definition. 

Suppose $A$ is a simple region with holes or a composite region. It is natural to call those simple regions appearing in the definition of a simple region with holes or a composite region `the atoms of $A$'. This is, however, not immediately clear for general plane regions. Take the complex region $C$ in Figure~\ref{fig:3comb} as an example. This region has two faces and one hole, which all happen to be simple regions. Another three simple regions can be obtained by taking the unions of adjacent faces and holes, \emph{viz.} $a_1\cup b_1,a_2\cup b_1, a_1\cup b_1\cup a_2$. Which simple region should be taken into account? Section~3 provides a detailed examination of this question. 


The following proposition shows that the envelope of $A$ is identical to the set complement of the unbounded exterior component of $A$. 
\begin{proposition}\label{prop:generalized-region}
Suppose $A$ is a bounded region, and $\widehat{A}$ is the envelope of $A$. Then $\widehat{A}$ is identical to the set complement of the unbounded exterior component of $A$. This implies that  $\widehat{A}$ has a connected exterior and no holes.
\end{proposition}
\begin{proof}
Write $U$ for the unbounded exterior component of $A$, and $V$ for the exterior of $\widehat{A}$. We assert that $U=V$. Because it is disjoint from any face or hole of $A$, $U$ is also disjoint from $\widehat{A}$. This implies that $U$ is contained  in $V$, which equals to the set complement of the closed set $\widehat{A}$. On the other hand, $V=\mathbb{R}^2\setminus \widehat{A}$ is also contained in $U$ because each point outside $\widehat{A}$ is contained in neither $A$ nor any of its holes. Hence $U=V$. Because $U$ is connected, we know $\widehat{A}$ has a connected exterior, and therefore has no holes. 
\end{proof}
A bounded plane region may have an infinite number of faces or holes.  Even worse, the intersection or union of two simple regions may have an infinite number of faces or holes. This makes general  plane regions impracticable for representing spatial objects in computers. For real-world applications, it is reasonable to require each plane region to be \emph{finitely} representable.

\section{Semi-Algebraic Regions}

In this section, we consider a very important class of plane regions that are finitely representable.


\begin{definition}
 A subset of the plane is called \emph{semi-algebraic} if it can be
defined by a Boolean combination of polynomial inequalities. A plane region is called a \emph{semi-algebraic region}  if it is a semi-algebraic set.
\end{definition}
For example, $(2\leq 2x^2+3y^2 \leq 3) \wedge (x-y=1)$ is a semi-algebraic set, but $ (0\leq x \leq 2) \wedge (0\leq y\leq \sin x)$  is not. Closed disks and polygons are all examples of semi-algebraic regions, but any region that has an infinite number of connected components is not a semi-algebraic region.  

Most sets in the plane that arise in practice can be closely approximated by semi-algebraic sets \cite{Basu06}, and semi-algebraic sets have been widely adopted in spatial database research (see \emph{e.g.} \cite{Benedikt+06,Papadimitriou+99}) and computational geometry (see \emph{e.g.} \cite{Basu06}).  Recently, it was proved \citep[Theorem 6]{KontchakovPZ10} that any satisfiable set of topological constraints has a solution using bounded semi-algebraic regions, where a topological constraint is either a binary RCC8 constraint or a unary connectedness constraint. 

In the following, we recall two properties of semi-algebraic sets that are related to the internal topological structure of plane regions. First, each semi-algebraic region is finitely representable in the following sense. 

\begin{proposition} (cf. \cite{Basu06,Coste00})
\label{prop:decomposition-sasets}
Let $A$ be an arbitrary semi-algebraic set of $\mathbb{R}^2$. Then $A$ can be decomposed as the disjoint union of finitely many pieces which are homeomorphic to either a singleton or the open unit interval $(0,1)$ or the open unit disk. As a consequence, $A$ has a finite number of connected components, each of which is semi-algebraic.
\end{proposition}
 
 Second, semi-algebraic sets are closed under set operations.
 \begin{proposition}(cf. \cite{Basu06,Coste00})
 \label{prop:closure-sasets}
 Semi-algebraic sets are closed under intersection, union, complement, and projection. Moreover, the topological closure, interior, and boundary of a semi-algebraic set are all semi-algebraic sets. \end{proposition}

In the remainder of this section, we give a close examination of components of bounded semi-algebraic regions. Suppose $A$ is a bounded semi-algebraic region. By Proposition~\ref{prop:closure-sasets}, we know $A^\circ$, $A^e$, and $\partial A$ are all bounded semi-algebraic sets, and, by Proposition~\ref{prop:decomposition-sasets}, $A$ has a finite number of interior or exterior components. 

Using these properties of semi-algebraic sets, we have the following characterizations for envelopes, faces and holes of semi-algebraic regions.
\begin{proposition}\label{prop:envelope-is-cr}
Let $A$ be a bounded semi-algebraic region. Then the envelope of $A$ is a composite region, and any face or hole of $A$ is a simple region with holes.
\end{proposition}
\begin{proof}
See Appendix A.
\end{proof}

Using the above results, we can characterize simple regions with holes and composite regions in terms of connectedness.
\begin{proposition} \label{prop:srh-con-interior}
A bounded semi-algebraic region is a simple region with holes (composite region, resp.) iff it has a connected interior (exterior, resp.).
\end{proposition}
\begin{proof}
Note that the `only if' part follows from the definitions of simple region with holes and composite region. We need only consider the `if' part.

Suppose $A$ is a bounded semi-algebraic region that has a connected interior. It is clear $A$ is its only face. By Proposition~\ref{prop:envelope-is-cr} we know $A$ is a simple region with holes. Suppose $A$ is a bounded semi-algebraic region that has a connected exterior. Because $A$ has no holes, we know $\widehat{A}$ is the union of all faces of $A$, and hence $A=\widehat{A}$. Therefore, by Proposition~\ref{prop:envelope-is-cr} we know $A$ is a composite region. 
\end{proof}
The above statements do not hold for general plane regions. For example, suppose $A,B$ are two  bounded connected regions such that $A\cup B$ is the unit closed disk, $A^\circ\cap B^\circ$ is empty, and $\partial A=\partial B$.\footnote{The existence of these $A,B$ is guaranteed by a theorem conceived by Brouwer. An informal description can be found in \url{http://www.cut-the-knot.org/do_you_know/brouwer.shtml}, also see \cite[Lemma~2.5]{Li06-ijgis}.} Then $A$ has only one face, \emph{i.e.} $A$ itself, and only one hole, \emph{i.e.} $B$. But $A$ is not a simple region with holes because the second condition of Definition~\ref{dfn:srh} is violated. Therefore, there are bounded regions with a connected interior that are not simple regions with holes.

We are now ready to define the atom set of a bounded semi-algebraic region.

\begin{definition} 
\label{dfn:atoms}
Let $o$ be a simple region. We say $o$ is an \emph{atom} of a simple region with holes $A$ if it is the envelope or a hole of $A$; $o$ is an \emph{atom} of a composite region $B$ if $o$ is a face of $B$; and say $o$ is an \emph{atom} of a bounded semi-algebraic region $C$ if $o$ is an atom of a face, a hole, or the envelope of $C$. For a bounded semi-algebraic region $X$, we write $\atom(X)$ for the set of all atoms of $X$.
\end{definition}

Take the three regions in Figure~\ref{fig:3comb} as examples. $A$ has a face and a hole, $B$ has two faces and no holes, and $C$ has two faces and one hole. Therefore, both $A$ and $B$ have two atoms, but $C$ has four atoms, \emph{viz.} $\overline{a_1}$, $\overline{a_2}$, $\overline{b_1}$, and $\widehat{C}$.

In the following section, we show atoms in $\atom(A)$ are necessary and sufficient to determine the global nine-intersection (topological) relations between $A$ and all other semi-algebraic regions.

\section{Local Characterization of the 9-Intersection Relations}

For two bounded semi-algebraic regions $A$ and $A^\prime$, the 9-intersection relation
\citep{EgenhoferH90} between $A$ and $A^\prime$ is defined as
\begin{equation}\label{eq:9int}
M(A,A^\prime)=\left(%
\begin{array}{ccc}
  A^\circ\cap {A^\prime}^\circ & A^\circ\cap\partial {A^\prime} & A^\circ\cap {A^\prime}^e \\
  \partial A\cap {A^\prime}^\circ & \partial A\cap\partial {A^\prime} &\partial A\cap {A^\prime}^e\\
  A^e\cap {A^\prime}^\circ & A^e\cap\partial {A^\prime} & A^e\cap {A^\prime}^e \\
\end{array}%
\right),
\end{equation}
by considering the content of the intersection (\emph{i.e.} whether it is empty or not). Among the $2^9=512$ possible 9-intersection relations, only eight are realizable between simple regions. These are exactly the RCC8 relations when restricted to simple regions (cf. Figure~\ref{fig:RCC8}). There are, however, 33 different 9-intersection relations if we consider bounded `non-exotic' plane regions \citep{SchneiderB06,Li06-ijgis}. These relations have not been implemented in the current GIS systems.

This section shows that the 9-intersection relations between two bounded semi-algebraic regions $A,A'$ can be derived from the 9-intersection relations between their atoms. Suppose $\atom(A)=\{o_1,\cdots,o_k\}$ and $\atom(A')=\{o_1',\cdots$, $o_l'\}$. This means that $M(A,A')$ can be derived from the RCC8 relations in the following $k\times l$ matrix
\begin{equation}\label{eq:9int-oo}
\left(%
\begin{array}{cccc}
  M(o_1,o_1') & M(o_1,o_2') & \cdots & M(o_1,o_l')  \\
  M(o_2,o_1') & M(o_2,o_2') & \cdots & M(o_2,o_l')  \\
\vdots & \vdots & \cdots & \vdots \\
  M(o_k,o_1') & M(o_k,o_2') & \cdots & M(o_k,o_l') 
\end{array}%
\right),
\end{equation}
where each $M(o_i,o_j')$ in Eq.\ref{eq:9int-oo} corresponds to a unique RCC8 relation \dc, \ec, \eq, \po, \tpp, \tppi, \ntpp, or \ntppi\ (cf. Figure~\ref{fig:RCC8}).

We show in Section~4.1 that these atoms are sufficient. Suppose an atom $o_i$ is removed from $\atom(A)$. We show in Section~4.2 there are two simple regions $o'$ and $o''$ such that $M(A,o')\not=M(A,o'')$ but $M(o_j,o')=M(o_j,o'')$ for all $o_j\in\atom(A)\setminus\{o_i\}$.  This implies that each atom $o_i$ in $\atom(A)$ is necessary.  

\subsection{The Sufficiency Theorem}

\begin{theorem}\label{thm:complete}
The 9-intersection relation between two bounded semi-algebraic regions $A,A^\prime$ can be
uniquely determined by the 9-intersection relations between the atoms of $A$ and
$A^\prime$.
\end{theorem}

To prove this theorem, we need only prove that each intersection  in Eq.~\ref{eq:9int} can be derived from the topological relations between the atoms of $A$ and $A'$. Note that the exterior-exterior intersection is always nonempty because $A$ and $A^\prime$ are bounded. For the boundary-boundary intersection we have
\begin{lemma}\label{lemma:partial2}
$\partial A\cap\partial A^\prime \not=\varnothing$ {\rm iff} there exist $a\in\atom(A)$ and $a'\in\atom(B)$ such that  $\partial
a\cap\partial a^\prime\not=\varnothing$.
\end{lemma}
\begin{proof}
This is because the boundary of a bounded semi-algebraic region is the union of the boundaries of all its atoms.
\end{proof}
It is clear that two simple regions $a,a'$ share a common boundary point iff they are related by either \ec, \po, \eq, \tpp\ or \tppi\ (see Figure~\ref{fig:RCC8}). Therefore, Lemma~\ref{lemma:partial2} can be restated to say that $A$ and $A'$ share a common boundary point iff they have atoms $a,a'$ that are related by either \ec, \po, \eq, \tpp\ or \tppi. In other words, we know $A$ and $A'$ share a common boundary point as long as one of the above five relations appears in Eq.~\ref{eq:9int-oo}.

As for the remaining seven intersections, we do not have a direct reduction as above. Instead, we first show (Lemmas~\ref{lemma:oo}-\ref{lemma:partial;e}) that each of these intersections can be uniquely reduced to checking 
\begin{itemize}
\item [(TC1)] Is the envelope of one region a subset of the envelope of the other region?
\item [(TC2)] Is a face/hole of one region a subset of the envelope of the other region?
\item [(TC3)] Does a face/hole of one region have a common interior point with a face/hole of another region?
\end{itemize}
Recall that each face/hole (the envelope, resp.) of a bounded semi-algebraic region is a simple region with holes (composite region, resp.) (cf. Proposition~\ref{prop:envelope-is-cr}). We then show these specific topological conditions between simple regions with holes and/or composite regions can be further reduced to the RCC8 relations between their atoms (Lemmas~\ref{lemma:crs-inclusion}-\ref{lemma:srhs-intersection}). In this way, we will derive the nine-intersections of $A$ and $A'$ by using the RCC8 relations between their atoms, and hence prove Theorem~\ref{thm:complete}.

We begin with the interior-interior intersection.
\begin{lemma}\label{lemma:oo}
$A^\circ\cap{A^\prime}^\circ\not=\varnothing$\ $\mathrm{iff}$\
there exist $a\in\face(A)$ and $a'\in\face(A')$ such that $a^\circ\cap a'^\circ\not=\varnothing$.
\end{lemma}
\begin{proof}
This is because the interior of a region is the union of the interiors of all its faces.
\end{proof}

We next consider the interior-boundary intersection. The case of boundary-interior intersection is similar.

\begin{lemma}\label{lemma:ob}
$A^\circ\cap\partial A^\prime\not=\varnothing$\ $\mathrm{iff}$\
 there exist $a\in\face(A)$, $a'\in\face(A')$, $b'\in\hole(A')$ such that $a^\circ\cap a'^\circ\not=\varnothing$, and $a\not\subseteq \widehat{A'}$ or $a^\circ\cap b'^\circ\not=\varnothing$.
\end{lemma}
\begin{proof}
Write $b_0'$ for the unbounded exterior component of $A'$. Recall that $a\not\subseteq \widehat{A'}$ iff $a^\circ\cap b_0'\not=\varnothing$ for $b_0'$ being the set complement of $\widehat{A'}$.

On one hand, suppose $A^\circ\cap\partial {A^\prime}\not=\varnothing$. There exists $a\in\face(A)$ such that $a^\circ\cap\partial A^\prime\not=\varnothing$. Take $P\in a^\circ\cap\partial A^\prime$. Because $P$ is a boundary point of $A'$ and $a^\circ$ is a neighborhood of $P$, we know $a^\circ$ contains an interior point of $A'$ as well as an exterior point of $A'$. This implies that
there exist $a^\prime\in\face(A')$ and $b^\prime\in\hole(A')$ such that $a^\circ\cap {a^\prime}^\circ\not=\varnothing$, and $a^\circ\cap b_0'\not=\varnothing$ or $a^\circ\cap {b^\prime}^\circ\not=\varnothing$. 

On the other hand, suppose $A^\circ\cap\partial {A^\prime}=\varnothing$. For any $a\in\face(A)$, we have $a^\circ\cap\partial{A^\prime}=\varnothing$. This implies that the connected open set $a^\circ$ is contained in the union of the interior and the exterior of $A^\prime$. This is possible iff $a^\circ$ is contained in a single connected component of either the interior of $A^\prime$ or the exterior of $A^\prime$. 
\end{proof}

We next consider the interior-exterior intersection. The case of exterior-interior intersection is similar.
\begin{lemma}\label{lemma:oe}
$A^\circ\cap {A^\prime}^e\not=\varnothing$\ $\mathrm{iff}$\
there exist $a\in\face(A)$, $b'\in\hole(A')$ such that $a\not\subseteq \widehat{A'}$ or $a^\circ \cap {b'}^\circ\not=\varnothing$.
\end{lemma}
\begin{proof}
Write $b_0'$ for the unbounded exterior component of $A'$.  Recall that $a\not\subseteq \widehat{A'}$ iff $a^\circ\cap b_0'\not=\varnothing$.
The statement holds because the interior of $A$ is the union of the interiors of all its faces, and the exterior of $A'$ is the union of $b_0'$ and all the interiors of $A'$'s holes.
\end{proof}

Lastly we consider the boundary-exterior intersection. The case of exterior-boundary intersection is similar.
\begin{lemma} \label{lemma:partial;e}
$\partial A\cap {A^\prime}^e\not=\varnothing$\ $\mathrm{iff}$\ $\widehat{A}\not\subseteq \widehat{A'}$ or there exist $a\in\face(A)$, $b\in\hole(A)$, and $b'\in\hole(A')$ such that $a^\circ\cap b'^\circ\not=\varnothing$, and $b'\not\subseteq\widehat{A}$ or $b^\circ\cap b'^\circ\not=\varnothing$.
\end{lemma}
\begin{proof}
Write $b_0'$ for the unbounded exterior component of $A'$. It is easy to see that $\partial A\cap {A^\prime}^e\not=\varnothing$ iff $\partial A\cap b_0'\not=\varnothing$ or $\partial A\cap b'^\circ\not=\varnothing$ for some hole $b'$ of $A'$. Note that $\widehat{A'}$, the envelope of $A'$, is equal to the set complement of $b_0'$. We know $\partial A\cap b_0'\not=\varnothing$ iff $\partial A\not\subseteq \widehat{A'}$. We next show that this is also equivalent  to saying $\widehat{A}\not\subseteq \widehat{A'}$. Suppose $\widehat{A}\not\subseteq\widehat{A'}$ but $\partial A\subseteq \widehat{A'}$. Then we have $\widehat{A}^\circ\not\subseteq \widehat{A'}^\circ$, \emph{i.e.} there exists $P$ in the interior of $\widehat{A}$ but not in the interior of  $\widehat{A'}$. Since $\widehat{A}$ and $\widehat{A'}$ are both bounded, there exists $Q$ that is in the exterior of both $\widehat{A}$ and $\widehat{A'}$. Because $\widehat{A'}$ has a connected exterior, we know there is a path $\alpha$ contained in the exterior of $\widehat{A'}$ that connects $P$ to $Q$. Because $P$ is in the interior of $\widehat{A}$ and $Q$ is in the exterior of $\widehat{A}$, there is a boundary point of $\widehat{A}$ on the path $\alpha$. This, however, contradicts the assumption that all boundary points of $A$ are contained in $\widehat{A'}$. Therefore, $\partial A\cap b_0'\not=\varnothing$ iff $\widehat{A}$ is contained in $\widehat{A'}$.

By an argument similar to that in the proof of Lemma~\ref{lemma:ob}, we can show that ``$\partial A\cap b'^\circ\not=\varnothing$" for a hole $b'$ of $A'$ is equivalent to saying ``there exist $a\in\face(A)$, $b\in\hole(A)$ such that $a^\circ\cap b'^\circ\not=\varnothing$, and $b\not\subseteq \widehat{A'}$ or $b^\circ\cap b'^\circ\not=\varnothing$".
\end{proof}



By the above lemmas, we know that the seven intersections other than boundary-boundary and exterior-exterior can be reduced to checking the three topological conditions (TC1)-(TC3) listed just before Lemma~\ref{lemma:oo}. Recall by Proposition~\ref{prop:envelope-is-cr} that each face/hole (the envelope) of a bounded semi-algebraic region is a simple region with holes (a composite region). These conditions can be rephrased as checking 
\begin{itemize}
\item [(TC1$'$)] Is a composite region a subset of  another composite region?
\item [(TC2$'$)] Is a  simple region with holes a subset of a composite region?
\item [(TC3$'$)] Do two simple regions with holes have a common interior point? 
\end{itemize}

We next show how these topological conditions can be derived from the RCC8 relations between their atoms. We begin with (TC1$'$).

\begin{lemma}\label{lemma:crs-inclusion}
Let $D$ and $D'$ be two composite regions with atoms  $d_1,\cdots,d_k$ and, respectively, $d_1',\cdots,d_l'$.
Then $D\subseteq D'$ iff each $d_i$ ($1\leq i\leq k$) is a subset of some $d_j'$ ($1\leq j\leq l$).
\end{lemma}
\begin{proof}
The `if' part is clear. For the `only if' part, suppose $D\subseteq D'$. Then each $d_i$ is contained in $D'$. As $d_i^\circ$ is a connected open set, it must be contained in one connected component of $D'^\circ$. That is, $d_i^\circ$ is a subset of some $d_j'^\circ$, which is possible iff $d_i\subseteq d_j'$.   
\end{proof}
This lemma shows that, if $D$ and $D'$ are two composite regions, then $D\subseteq D'$ iff for any $1\leq i\leq k$ there exists $1\leq j\leq l$ such that $\eq(d_i,d_j') \vee \tpp(d_i,d_j') \vee \ntpp(d_i,d_j')$ holds. We next consider (TC2$'$).

\begin{lemma}\label{lemma:srh-contained-in-cr}
Let $D=(d_0;d_1,\cdots,d_k)$ be a simple region with holes and $D'$ be a composite region with atoms $d_1^\prime,d_2^\prime,\cdots,d_l^\prime$.
Then $D\subseteq D'$ iff $d_0$ is contained in some $d_i^\prime$ $(1\leq i\leq l)$.
\end{lemma}
\begin{proof}
Suppose $D\subseteq D'$. As a connected set $D^\circ$ is contained in a connected component of $D'^\circ$. That is, $D^\circ$ is contained in the interior of some $d_i^\prime$. Therefore, $D$ itself is contained in $d_i^\prime$. It is easy to see that any hole of $D$ is also contained in $d_i^\prime$. As a consequence, we know $d_0\subseteq d_i'$. On the other hand, suppose $d_0$ is contained in some $d_i^\prime$. We know $D\subseteq d_i^\prime\subseteq D'$ because $D$ is contained in $d_0$.
\end{proof}

This lemma shows that if $D$ is a simple region with holes and $D'$ is a composite region, then $D\subseteq D'$ iff  $\eq(d_0,d_j')\vee \tpp(d_0,d_j') \vee \ntpp(d_0,d_j')$ holds for some $1\leq j\leq l$.
As for (TC3$'$), we have
\begin{lemma}\label{lemma:srhs-intersection}
Let $D=(d_0;d_1,\cdots,d_k)$ and
$D^\prime=(d_0^\prime;d_1^\prime,\cdots,d_l^\prime)$ be two simple regions with holes.
Then $D^\circ\cap {D^\prime}^\circ=\varnothing$ iff  $d_0^\circ\cap
{d_0^\prime}^\circ=\varnothing$, or $d_0\subseteq d_i^\prime$ for
some $i$, or $d_0^\prime\subseteq d_j$ for some $j$.
\end{lemma}
\begin{proof}
The `if' part is clear. We prove the `only if' part. Let $b_0$ ($b_0'$) be the unbounded exterior component of $D$ ($D'$). Then $b_0$ ($b_0'$) is the set complement of $d_0$ ($d_0'$). 

Suppose $D^\circ\cap {D^\prime}^\circ=\varnothing$ and $d_0\not\subseteq d_i^\prime$ for
any $i$, and $d_0^\prime\not\subseteq d_j$ for any $j$. We show $d_0^\circ\cap {d_0^\prime}^\circ=\varnothing$. By  $D^\circ\cap {D^\prime}^\circ=\varnothing$ we know  $D^\circ\cap {D^\prime}=\varnothing$ is also true. This implies that $D^\circ$ is contained in the exterior of $D'$. Because $D^\circ$ is a connected open set, it must be contained in a connected component of the exterior of $D'$, which is either $d_j'^\circ$ $(1\leq j\leq l)$ or $b_0'$. 

If $D^\circ$ is contained in $d_j'^\circ$ for some $1\leq j\leq l$, then each hole of $D$ is also contained in $d_j'$. Therefore $d_0^\circ$ is contained in $d_j'^\circ$. This contradicts our assumption  that $d_0^\circ\not\subseteq d_j'^\circ$ for $1\leq j\leq l$.  So $D^\circ$ should be contained in $b_0'$. In this case we have $D^\circ\cap d_0'=\varnothing$, and hence $D\cap d_0'^\circ=\varnothing$. This implies that $d_0'^\circ$ is contained in a connected component of the exterior of $D$. By our assumption, $d_0'$ is not contained in any $d_i$ for $1\leq i\leq k$. Therefore, $d_0'^\circ$ is not contained in $d_i^\circ$ for $1\leq i\leq k$. So $d_0'^\circ$ must be contained in $b_0$. Hence $d_0'^\circ\cap d_0=\varnothing$ and $d_0^\circ\cap d_0'^\circ=\varnothing$ for $d_0$ being the set complement of $b_0$.
\end{proof}
The above lemma asserts that two simple regions with holes have no common interior points iff either $\dc(d_0,d_0') \vee \ec(d_0,d_0')$, or $\eq(d_0,d_i') \vee \tpp(d_0,d_i') \vee \ntpp(d_0,d_i')$ for some $1\leq i\leq l$, or $\eq(d_0',d_j) \vee \tpp(d_0',d_j) \vee \ntpp(d_0',d_j)$ for some $1\leq j\leq k$.

By the above three lemmas, we know the three topological conditions TC1-TC3 can be determined by the RCC8 relations between the atoms of the two regions. Applying these results to the intersections between $A,A'$ other than boundary-boundary and exterior-exterior, we can compute the intersections of $A$ and $A'$ from the matrix in Eq.~\ref{eq:9int-oo}. 

Take the interior-boundary intersection as an example. To determine if $A^\circ\cap\partial A'$ is nonempty, by Lemma~\ref{lemma:ob} we need only check if there exist $a\in\face(A)$, $a'\in\face(A')$, $b'\in\hole(A')$ such that $a^\circ\cap a'^\circ\not=\varnothing$, and $a\not\subseteq \widehat{A'}$ or $a^\circ\cap b'^\circ\not=\varnothing$. Choose any $a,a'$ and $b'$ as above. We check if they satisfy the above conditions. If the answer is affirmative, then we know $A^\circ\cap\partial A'$ is nonempty. If the answer is negative, we choose another triple of faces/holes, and continue this way until all possible triples have been checked. In that case we have $A^\circ\cap\partial A'=\varnothing$.  

We start by checking $a\not\subseteq \widehat{A'}$ for all $a\in\face(A)$. By Lemma~\ref{lemma:srh-contained-in-cr} we need only check if the envelope of $a$ is contained in an atom of $\widehat{A'}$. To this end, we need list all atoms of $\widehat{A'}$ and identify the envelope of $a$. Suppose $o_i$ is the envelope of $a$ and $o'_{j_1},o'_{j_2},\cdots,o'_{j_k}$ are atoms of $\widehat{A'}$. We conclude that $a\subseteq \widehat{A'}$ when $M(o_i,o'_{j_s})$ is either \eq, or \tpp, or \ntpp\ for some $1\leq s\leq k$; and  $a\subseteq \widehat{A'}$ otherwise. We then check if $a^\circ\cap a'^\circ\not=\varnothing$. Suppose the two faces are $a=(d_0;d_1,\cdots,d_k)$ and $a^\prime=(d_0^\prime;d_1^\prime,\cdots,d_l^\prime)$. By Lemma~\ref{lemma:srhs-intersection} we know 
$a^\circ\cap {a^\prime}^\circ=\varnothing$ iff  $d_0^\circ\cap {d_0^\prime}^\circ=\varnothing$, or $d_0\subseteq d_i^\prime$ for some $i$, or $d_0^\prime\subseteq d_j$ for some $j$.
Because all $d_i$ are atoms of $A$, and all $d_j'$ are atoms of $A'$, the equations $d_0^\circ\cap {d_0^\prime}^\circ=\varnothing$, $d_0\subseteq d_i^\prime$, and $d_0^\prime\subseteq d_j$ can be easily determined by using the matrix in Eq.~\ref{eq:9int-oo}. Whether $a^\circ\cap b'^\circ\not=\varnothing$ can be determined similarly.

Combined with Lemma~\ref{lemma:partial2}, we conclude that the nine intersections of two bounded semi-algebraic regions $A,A^\prime$ can be derived from the RCC8 relations between atoms of $A$ and $A^\prime$ given in Eq.~\ref{eq:9int-oo}. Thus we have proved Theorem~\ref{thm:complete}.

A special case of Theorem~\ref{thm:complete} has been investigated in \cite{McKenneyPS08} for simple regions with holes. 
\begin{corollary}\label{coro:srh-local}
Let $A=(d_0;d_1,\cdots,d_k)$ and $A^\prime=(d_0^\prime; d_1^\prime,\cdots,d_l^\prime)$ be two simple regions with holes. Then the 9-intersection relation between $A$ and $A^\prime$ is completely determined by the RCC8 relations between all $d_i$ and all $d_j^\prime$.
\end{corollary}

\subsection{The Necessary Theorem}

Theorem~\ref{thm:complete} shows that these atoms are sufficient to determine the 9-intersection relation between any two bounded semi-algebraic regions  $A$ and $A^\prime$. If $A'$ is a simple region, then we have in particular the following result.
\begin{proposition}\label{prop:suf4atoms}
Suppose $A$ is a bounded semi-algebraic region, and $d$ and $d^\prime$ are two simple regions. Then we have
\begin{equation}\label{eq:suf4atoms}
M(A,d)=M(A,d^\prime)  \ \mbox{if}\  M(o,d)=M(o,d^\prime)\ \mbox{holds\ for\ all}\ o\in\atom(A),
\end{equation}
where $M(X,Y)$ is the 9-intersection relation between $X$ and $Y$ (see Eq.~\ref{eq:9int}).
\end{proposition}
Are these atoms \emph{necessary}? That is,
do we need all these atoms of $A$ to compute the 9-intersection relation between $A$ and any other bounded semi-algebraic region $A^\prime$? What happens if we delete an atom $o$ from the atom set $\atom(A)$? In this subsection we
investigate these problems.

We begin with an example. Consider the region $A$ in Figure~\ref{fig:co1+}(a), which has two faces $a_1,a_2$ and one hole $b_1$. As usual, we write $b_0$ for the unbounded exterior component of $A$. Note that $A$ is neither a simple region with holes nor a composite region. The atom set of $A$ is $\atom(A)=\{a_1,a_2,b_1,\widehat{A}\}$, where $\widehat{A}=a_1\cup a_2\cup b_1$. We next show that all atoms of $A$ are necessary. To this end, we show, for each atom $d$ of $A$, there exists a simple region $d^\prime$ such that
\begin{equation}
\mbox{$M(A,d)\not=M(A,d^\prime)$ but $M(o,d)=M(o,d^\prime)$ for all
$A$-atoms $o\not=d$.}
\end{equation}

Consider the hole $b_1$. Let $b^\prime$ be the simple region as
shown in Figure~\ref{fig:co1+}(b). Then $b^\prime$ has the same RCC8
relations with $a_1,a_2$ and $\widehat{A}$ as $b_1$ does. Note that $\partial b_1\subset\partial
A$ while $\partial b^\prime\not\subseteq\partial A$. This implies that $ A^\circ\cap \partial b_1=\varnothing$ and $ A^e\cap \partial b_1=\varnothing$ but either $ A^\circ\cap \partial b'\not=\varnothing$ or $A^e\cap \partial b'\not=\varnothing$. Therefore, $M(A,b_1)\not=M(A,b^\prime)$.

Take $\widehat{A}$ as another example. Let  $b^\ast$ be the simple region that contains $\widehat{A}$ as shown in Figure~\ref{fig:co1+}(c). Then $\widehat{A}$ and $b^\ast$ have the same 9-intersection
relations with $a_1,a_2$ and $b_1$. But since $\partial
\widehat{A}\subset\partial A$ and $\partial b^\ast\not\subseteq\partial A$, we
know $A^\circ\cap\partial\widehat{A}=\varnothing$ and $A^e\cap\partial\widehat{A}=\varnothing$ but either $A^\circ\cap\partial b^\ast\not=\varnothing$ or $A^e\cap\partial b^\ast\not=\varnothing$. Therefore,  $M(\widehat{A},A)\not=M(b^\ast,A)$.

Similar construction can be applied to $a_1$ and $a_2$.

\begin{figure}[h]
\centering
\begin{tabular}{ccc}
\includegraphics[width=.3\textwidth]{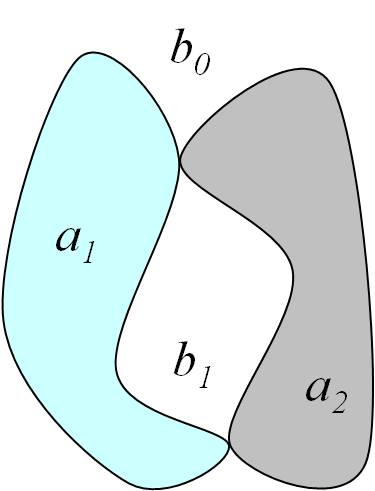}
&\includegraphics[width=.35\textwidth]{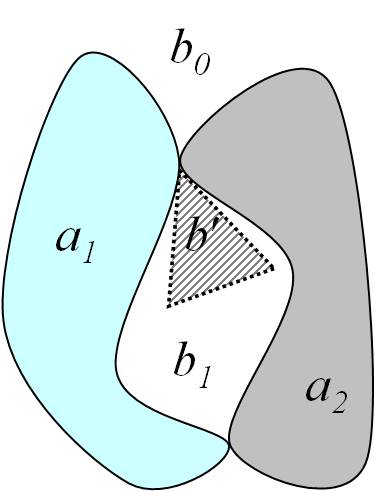} &
\includegraphics[width=.3\textwidth]{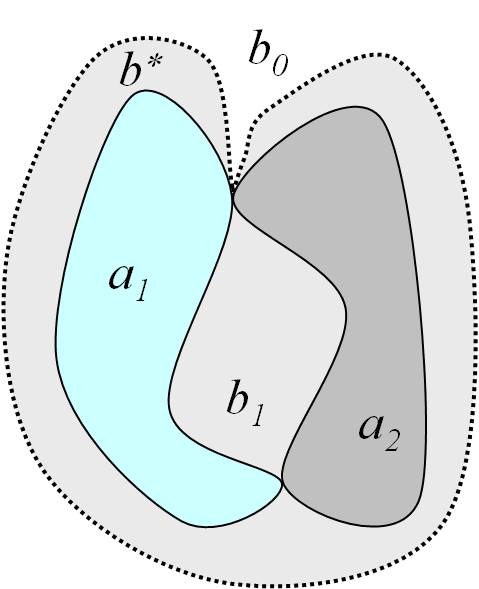}\\
(a) & (b) & (c)
\end{tabular}
\caption{A bounded region and its atoms.} \label{fig:co1+}
\end{figure}

The following lemma shows that the above conclusion holds for any bounded semi-algebraic region.
\begin{lemma}\label{thm:local-nece}
Suppose $A$ is a bounded semi-algebraic region and $d$ is an atom of $A$. Then there exists another simple region $d^\prime$ such that
$M(A,d)\not=M(A,d^\prime)$ but $M(o,d)=M(o,d^\prime)$ for each
atom $o\not=d$ of $A$.
\end{lemma}
\begin{proof}
Note that $d\in\atom(A)$ implies $\partial d\subseteq\partial A$. We construct a simple region $d^\prime$ that is similar to $d$ but
$\partial d^\prime\not\subseteq\partial A$. Note that because $A^\circ\cap\partial d=\varnothing$ and $A^e\cap\partial d=\varnothing$ but either $A^\circ\cap\partial d'\not=\varnothing$ or $A^e\cap\partial d'\not=\varnothing$, we will have $M(A,d)\not=M(A,d^\prime)$.

We take the case when $d$ is the envelope of a face/hole $c$ as an example. Without lack of generality, we assume that $c$ itself has at least one hole.
For each $o\in\atom(A)$, take a point
$P_o$ in $\partial d\cap\partial o$ if it is nonempty. There exists a Jordan arc $\alpha$ contained in the Jordan
curve $\partial d$ that is disjoint from all holes of $c$. Take
a point $P\in\alpha$ such that $P\not=P_o$ for all $o$. So the
distance from $P$ to the holes of $c$, denoted by $\delta$, is nonzero. Write $B^\circ(P,\varepsilon)$ be the open disk centred at $P$ with radius $\varepsilon\ll \delta$, and let $d'$ for $d\setminus B^\circ(P,\varepsilon)$. Then $d'$ is still
simply connected and contains all holes of $d$. Clearly, there is a point $Q$ in
the boundary of $d^\prime$ that is not on $\partial A$. This implies that $M(A,d)\not=M(A,d^\prime)$. On the other hand, for
each atom $o\not=d$ of $A$, if $\partial o\cap\partial d$ is
empty, so is $\partial o\cap\partial d^\prime$; and if $\partial
o\cap\partial d$ is nonempty, then $P_o\in\partial
o\cap\partial d^\prime\not=\varnothing$. Moreover, for each $o\in\atom(A)$ with $o\not=d$, the part-whole
relation between $d^\prime$ and $o$ is the same as that
between $d$ and $o$. Therefore, $M(o,d)=M(o,d^\prime)$ for
all atoms $o\not=d$ of $A$.

The case when $d$ is a hole of a face/hole $c$ or a face of $\widehat{A}$ is similar.
\end{proof}

As a consequence, we know that each atom in $\atom(A)$ is necessary.

\begin{theorem}\label{thm:nece}
Each atom of a bounded semi-algebraic region $A$ is necessary in locally determining the 9-intersection relation of $A$ and other bounded semi-algebraic regions.
\end{theorem}

\section{A Graph Representation of Plane Regions}

In this section we first introduce  for each bounded semi-algebraic region $A$ a unique link graph to represent the internal structure of $A$, and then give methods for computing the atoms of $A$.



\subsection{Link Graph}
We recall that a \emph{Jordan arc} is a plane subset that is homeomorphic to the closed interval
$[0,1]$. The boundary of each simple region contains a Jordan arc.
\begin{definition}\label{dfn:linked-region}
Two open sets $U,V$ in $\mathbb{R}^2$ are \emph{linked} if $U\cap V$ is empty and $\partial U\cap \partial V$ contains a Jordan arc.
\end{definition}

\begin{lemma}
\label{lemma:union}
Suppose $A,B$ are two bounded semi-algebraic regions with a connected interior. If $A^\circ$ and $B^\circ$ are linked, then $A\cup B$ has a connected interior.
\end{lemma}
\begin{proof}
Suppose $A^\circ$ and $B^\circ$ are linked by a Jordan arc $\varphi$ contained in $\partial A\cap \partial B$. Let $P_1,P_2$ be the two endpoints of $\varphi$. Take a point $P$ in $\varphi$ different from $P_1,P_2$. Because $\psi=\{P_1,P_2\}\cup (\partial A\setminus{\varphi})$ is a closed set that does not contain $P$, we know $d(P,\psi)>0$. Take $\delta<d(P,\psi)$. Then we can show $B^\circ(P,\delta)$ is contained in $A\cup B$. Since $B^\circ(P,\delta)$ contains an interior point of $A$ and an interior point of $B$, we know  $A\cup B$ has a connected interior.
\end{proof}

By Proposition~\ref{prop:srh-con-interior} we know a bounded semi-algebraic region $X$ is a simple region with holes iff $X^\circ$ is connected. Therefore, we have
\begin{corollary}
\label{coro: x}
Suppose $A,B$ are two simple regions with holes, if $A^\circ$ and $B^\circ$ are linked, then $A\cup B$ is a simple region with holes.
\end{corollary}

Recall that a component of a plane region $A$ is a connected component of either the interior or the exterior of $A$.
\begin{proposition}
Suppose $A$ is a bounded semi-algebraic region. Then each component of $A$ is linked to at least one other component of $A$. Moreover, if two components are linked, then exactly one is an interior component of $A$.
\end{proposition}
\begin{proof}
Because the boundary of each component $c$ is contained in the union of the boundaries of
all other components, we know that $c$ is linked to at least one other component. Suppose $c_1,c_2$ are two linked components. By Lemma~\ref{lemma:union} we know that $(\overline{c_1}\cup \overline{c_2})^\circ$ is connected. If both $c_1$ and $c_2$ are interior components, then $c_1\cup c_2 \subseteq (\overline{c_1}\cup\overline{c_2})^\circ$ is contained in $A^\circ$, which contradicts the assumption that $c_1$ is a connected component of $A^\circ$. Therefore, at most one of $c_1$ and $c_2$ is an interior component. Similarly, we can show at most one of $c_1$ and $c_2$ is an exterior component.
\end{proof}

For a bounded semi-algebraic region $A$, we introduce a level function, which classifies the
components of $A$ into levels.
\begin{definition}
Let $A$ be a bounded semi-algebraic region. For each component $c$ of $A$, we define
$\lev(c)$, the level of component $c$, inductively as follows:
\begin{itemize}
\item The level of $b_0$, the unbounded exterior component of $A$, is 0;
\item For an undefined $c^\ast$, if there exists a previously defined node
$c$ which is linked to $c^\ast$, then define
$\lev(c^\ast)=\lev(c)+1$.
\end{itemize}
That is, $\lev(c)$ is the minimum length of a sequence $c=c_n,c_{n-1},\cdots,b_0$ whose successive elements are linked.
\end{definition}
For two linked components $c_1,c_2$, it is easy to prove that
$\lev(c_1)-\lev(c_2)=\pm 1$. Take the bounded semi-algebraic region $C$ in Figure~\ref{fig:3comb} as an example. We have
$\lev(b_0)=0$, $\lev(a_1)=\lev(a_2)=1$, and $\lev(b_1)=2$.

\begin{definition}\label{dfn:linkgraph}
The \emph{link graph} $\mathcal{G}_A$ of a bounded semi-algebraic region $A$ is defined as
a directed graph $(N(A),E(A))$ as follows:
\begin{itemize}
\item $N(A)$ is the set of all connected components of $A^\circ$ and $A^e$;
\item For $c_1,c_2\in N(A)$, $(c_1,c_2)\in E(A)$ if they are linked
and $\lev(c_2)=\lev(c_1)+1$.
\end{itemize}
In other words, the link graph is naturally bipartite. Because $b_0$ has level 0 in the layered graph $\mathcal{G}_A$, we call $b_0$ the \emph{root} of $\mathcal{G}_A$. If there is a directed edge from $c_1$ to $c_2$, then we call $c_1$ a \emph{parent} of $c_2$, and call $c_2$ a \emph{child} of $c_1$. 
\end{definition}

In this way, we associate each bounded semi-algebraic region with a unique layered graph. Figure~\ref{fig:graph-b} shows the link graphs of regions $A,B,C$ in Figure~\ref{fig:3comb}. 



\begin{figure}
\centering
\begin{tabular}{ccc}
\includegraphics[width=.065\textwidth]{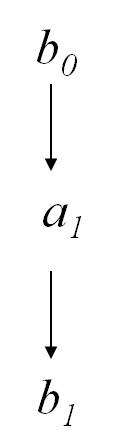}
&
\includegraphics[width=.2\textwidth]{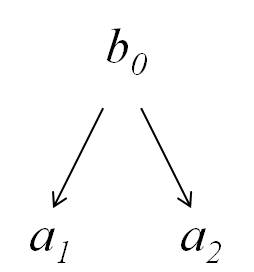}
&
\includegraphics[width=.2\textwidth]{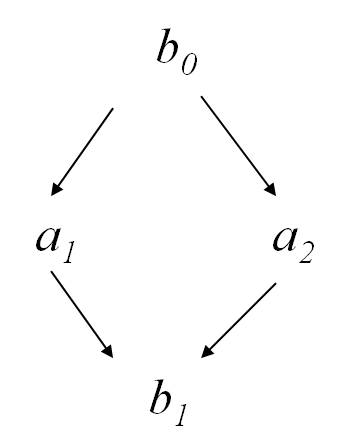}\\
(a) & (b) & (c)
\end{tabular}
\caption{The link graphs of regions in Figure~\ref{fig:3comb}.}
\label{fig:graph-b}
\end{figure}

In the next proposition, we give characterizations of simple regions with holes and
composite regions in terms of their link graphs.
\begin{proposition}
Let $A$ be a bounded semi-algebraic region. Then $A$ is a simple region with holes
iff its link graph $\mathcal{G}_A$ has a unique node at level 1 and
has no node with level greater than 2; and $A$ is a composite region iff its link graph $\mathcal{G}_A$ has no node with level greater than 1.
\end{proposition}

In other words, a bounded semi-algebraic region is a simple region with $k$ holes iff its link graph has the
form as the one given in Figure~\ref{fig:graph-sr}(a); and a bounded semi-algebraic region is a composite region (with $k$ atoms) iff its link
graph has the form as the one given in Figure~\ref{fig:graph-sr}(b).

\begin{figure}
\centering
\begin{tabular}{ccc}
\includegraphics[width=.22\textwidth]{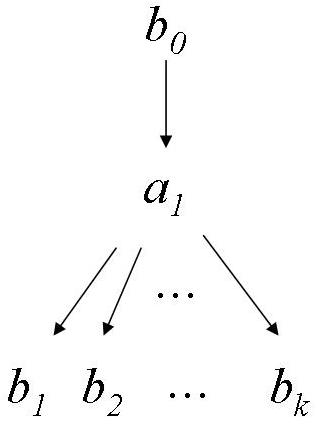}
&
&
\includegraphics[width=.3\textwidth]{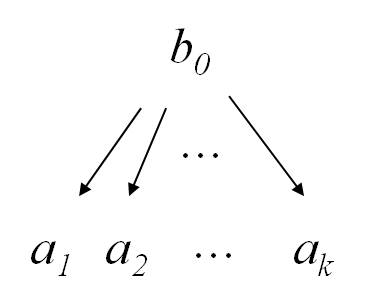}\\
(a) && (b)
\end{tabular}
\caption{The link graphs of (a) a simple region with holes with $k$ holes and (b) a composite region with $k$ atoms}
\label{fig:graph-sr}
\end{figure}

The link graphs of bounded semi-algebraic regions can be used to approximately determine if two semi-algebraic regions are homeomorphic.
\begin{theorem}
Let $A,B$ be two bounded semi-algebraic regions. If the link graph of $A$ is not isomorphic to the link graph of $B$, then $A$ is not homeomorphic to $B$.
\end{theorem}
For example, the regions in  Figure~\ref{fig:3comb} are pairwise non-homeomorphic. This is because
their link graphs are pairwise non-isomorphic, as shown in Figure~\ref{fig:graph-b}.

\subsection{Computing the Atoms of $A$ from the Link Graph}

In this subsection, we suppose that $A$ is a semi-algebraic region and $\mathcal{G}_A$ is its link graph. We show how to compute the atoms of $A$ from $\mathcal{G}_A$. Recall that each atom of $A$ is either an atom of the envelope $\widehat{A}$, or an atom of a face or hole of $A$. 

\begin{proposition}
\label{prop:atom-comp}
Each atom of $A$ is the closure of the union of several connected components of $A^\circ\cup A^e$.
\end{proposition}
\begin{proof}
Suppose $o$ is an atom  of $A$. Then we have $\partial o\subseteq \partial A$. Let $c$ be a bounded connected component of $A^\circ$ or $A^e$. Because $c$ cannot be separated by the boundary of $o$, we know either $o^\circ\cap c=\varnothing$ or $c\subseteq o^\circ$. Therefore, $o$ is the closure of the union of several components of $A^\circ\cup A^e$.
\end{proof}

We first consider the atoms of $\widehat{A}$.


\begin{proposition}\label{prop:generegion}
Suppose $\mathcal{G}_A$ (as an undirected graph) is decomposed into $k$ connected subgraphs
$\mathcal{G}_1(c),\cdots,\mathcal{G}_k(c)$ after the root node $b_0$ has been removed from $\mathcal{G}_A$. The closure of the union of the bounded components 
in each $\mathcal{G}_i(c)$ is an atom of the composite region $\widehat{A}$.
\end{proposition}
\begin{proof}
Suppose $c_1$ and $c_2$ are linked. Then by Lemma~\ref{lemma:union} $\overline{c_1}\cup \overline{c_2}$ has a connected interior. In general, suppose $c_1,c_2,\cdots,c_k$ is a path in the undirected graph $\mathcal{G}_A$, then $\overline{c_1\cup c_2\cup \cdots \cup c_k}=\overline{c_1}\cup \overline{c_2}\cup\cdots\cup \overline{c_k}$ has a connected interior. Because each $\mathcal{G}_i(c)$ is a connected subgraph of $\mathcal{G}(A)$, we know the interior of $\overline{\bigcup\{c': c'\in \mathcal{G}_i(c)\}}$ is connected. Write $U$ for this connected open set. Suppose the envelope $\widehat{A}$ has atoms $o_1,o_2,\cdots, o_s$. Because $o_1^\circ,o_2^\circ, \cdots, o_s^\circ$ are the connected components of the interior of $\widehat{A}$. We know $U$ is contained in $o_i^\circ$ for some $1\leq i\leq s$. 

On the other hand, by Proposition~\ref{prop:atom-comp} we know $o_i$ is the closure of the components of $A$ it contains. We show $o_i$ contains exactly those components in $\mathcal{G}_i(c)$.  In fact, suppose $c_1,c_2$ are two components contained in $o_i$. We show there is a path in $\mathcal{G}(A)$ that connects $c_1$ and $c_2$. By Proposition~\ref{prop:decomposition-sasets}, we suppose $\partial A$ is the disjoint union of $k$ points in $\mathbf{N}=\{P_1,\cdots,P_k\}$  and $l$ open arcs in $\mathbf{E}=\{\varphi_1,\cdots,\varphi_l\}$, where each $\varphi_j$ has two endpoints, both are in $\mathbf{N}$, and each point $P_i$ is an endpoint of some arc in $E$.  For any $Q_1\in c_1$ and $Q_2\in c_2$, there is a path $\alpha$ in the interior of $o_i$ which is disjoint from $\mathbf{N}$. Because $\alpha$ is compact, we can find a finite set of open disks centred at points in $\alpha$ which covers $\alpha$ and each disk intersects at most two linked components. It is now clear that there is a path in $\mathcal{G}(A)$ that connects $c_1$ and $c_2$.   

Therefore, each atom $o_i$ is the closure of the union of the bounded components in some $\mathcal{G}_i(c)$. 
\end{proof}

We next compute the envelope and holes of a bounded component of $A$.
\begin{proposition}\label{thm:hole}
Suppose $c$ is a bounded component of $A$ and $\mathcal{G}_A$ (as an undirected graph) is decomposed into $k+1$ connected subgraphs $\mathcal{G}_0(c),\mathcal{G}_1(c),\cdots,\mathcal{G}_k(c)$ after the node $c$ has been removed from $\mathcal{G}_A$, where $\mathcal{G}_0(c)$ is the connected subgraph that contains $b_0$. Then a component 
$c_1\not=c$ is in $\mathcal{G}_0(c)$ iff there is a path from $b_0$ to $c_1$ that does not pass through $c$. Moreover, for each $i\geq 1$, the closure of the union of the components in $\mathcal{G}_i(c)$ is a hole of $c$, and the envelope of $c$ is the closure of the union of $c$ and all its holes.
\end{proposition}
\begin{proof}
If $c_1\not=c$ and there is a path from $b_0$ to $c_1$, then by definition $c_1$ is a node in $\mathcal{G}_0(c)$. On the other hand, suppose $c_1\not=c$ and any path from $b_0$ to $c_1$ passes $c$. Then $c$ must be in a connected subgraph different from $\mathcal{G}_0(c)$. 

Because $c$ is a bounded component of $A$, we know its closure is a simple region with holes. Write $\overline{c}=(d_0;d_1,\cdots,d_s)$. We show $s=k$ and each $d_i$ is the closure of the union of the components in a unique $\mathcal{G}_i(c)$. We first note that each component not in $\mathcal{G}_0(c)$ is contained in some $d_i$ ($i>0$), and each $d_i$ is the closure of the union of several components which are not in $\mathcal{G}_0(c)$.

As in the proof of Proposition~\ref{prop:generegion}, we can show that (i) the closure of the union of the components in each $\mathcal{G}_i(c)$ has a connected interior and hence is contained in some $d_i$ ($i>0$); (ii) for any two components $c_1,c_2$ contained in $d_i$ ($i>0$), there is a path from $c_1$ to $c_2$ in $\mathcal{G}(A)$ whose nodes are all components contained in $d_i$. This guarantees that each $d_i$ ($i>0$) is the closure of the union of components in some unique $\mathcal{G}_j(c)$ ($j>0$). It is then clear that $s=k$, and $d_0$ is the union of the closure of $c$ and all holes $d_i$ ($i>0$).
\end{proof}

By the above result, it is clear that a bounded component $c$ has no holes iff  $\mathcal{G}_A$ is still connected after $c$ is removed. We end this section with an example.

\begin{figure}[htb]
\centering
\begin{tabular}{ccc}
\includegraphics[width=.36\textwidth]{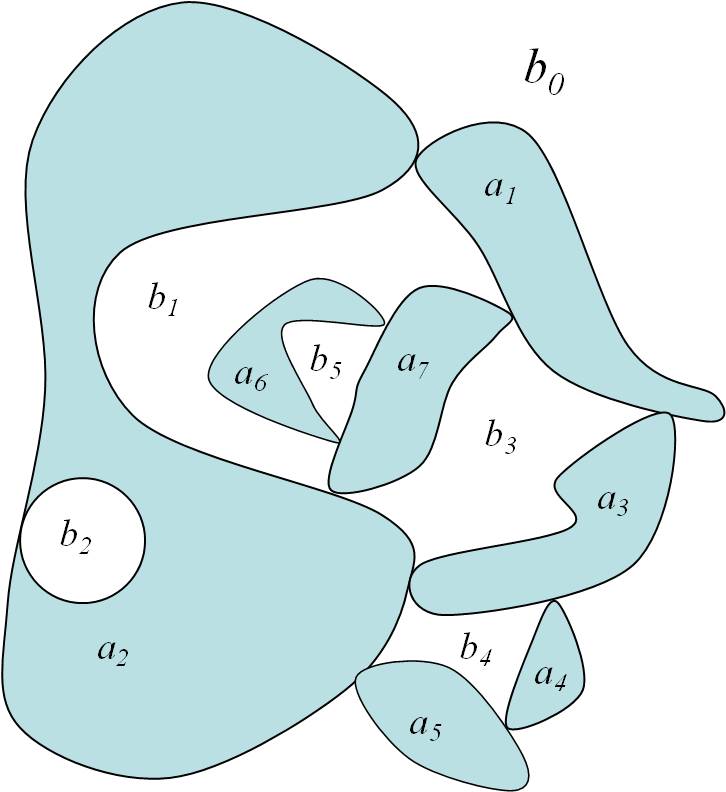}
&
\includegraphics[width=.3\textwidth]{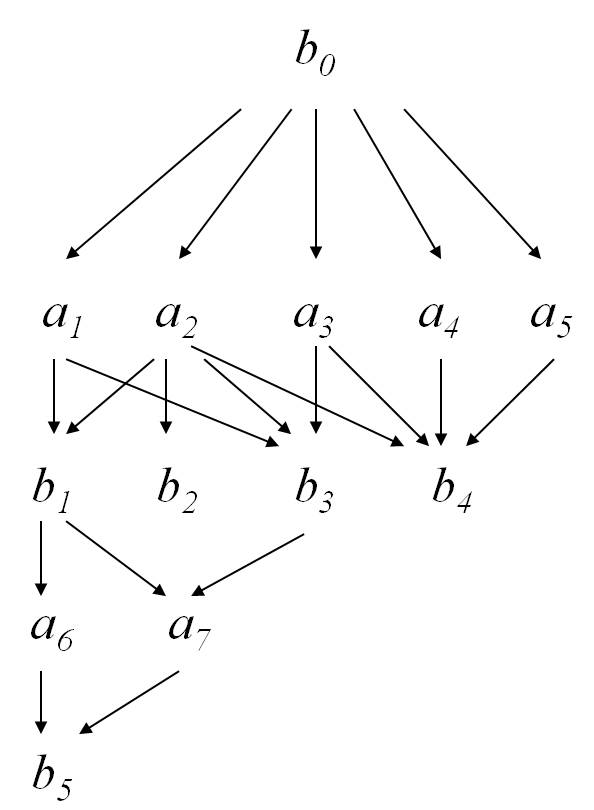}
&
\includegraphics[width=.2\textwidth]{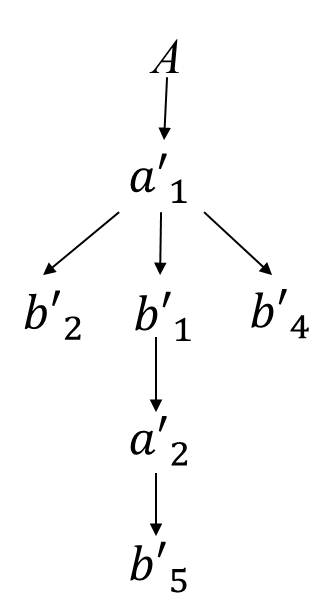}
\end{tabular}
\caption{A bounded semi-algebraic region $A$ (left), its link graph (center), and its W-B tree (right), where $a_1'$ is the envelope of $A$, $a_2'$ is the closure of $a_6\cup a_7$, $b_i'$ is the closure of $b_i$ for $i=2,4,5$, and $b_1'$ is the closure of $b_1\cup b_3\cup b_5\cup a_6\cup a_7$.} \label{fig:co2}
\end{figure}

\begin{example}\label{exam:co2}
Figure~\ref{fig:co2} shows a bounded semi-algebraic region and its link
graph, where an arrow from $c_1$ to $c_2$ suggests that the level of
$c_1$ is lower than that of $c_2$. It is
easy to see that only $a_2$ has a hole, and the closures of the other bounded components of this region
are all simple regions.
Figure~\ref{fig:subgraphs} shows the connected subgraphs of $\mathcal{G}_A$ obtained after $a_2$, $a_3$, and $b_0$ are removed, respectively, from $\mathcal{G}_A$.
\end{example}

\begin{figure}[htb]
\centering
\begin{tabular}{ccc}
\includegraphics[width=.3\textwidth]{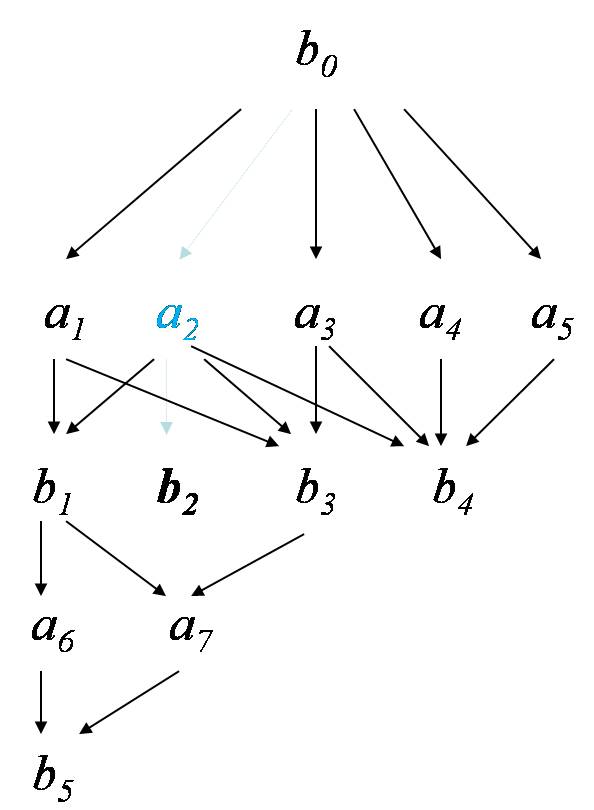}
&
\includegraphics[width=.3\textwidth]{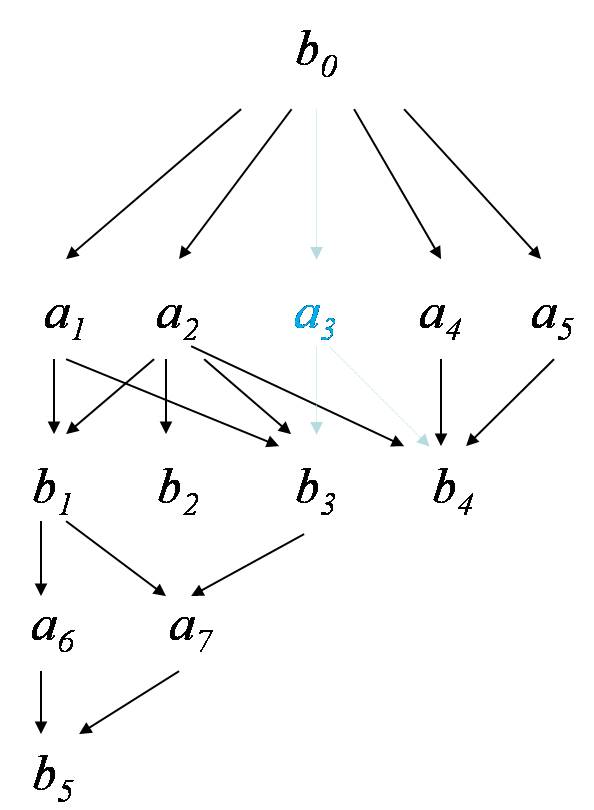}
&
\includegraphics[width=.3\textwidth]{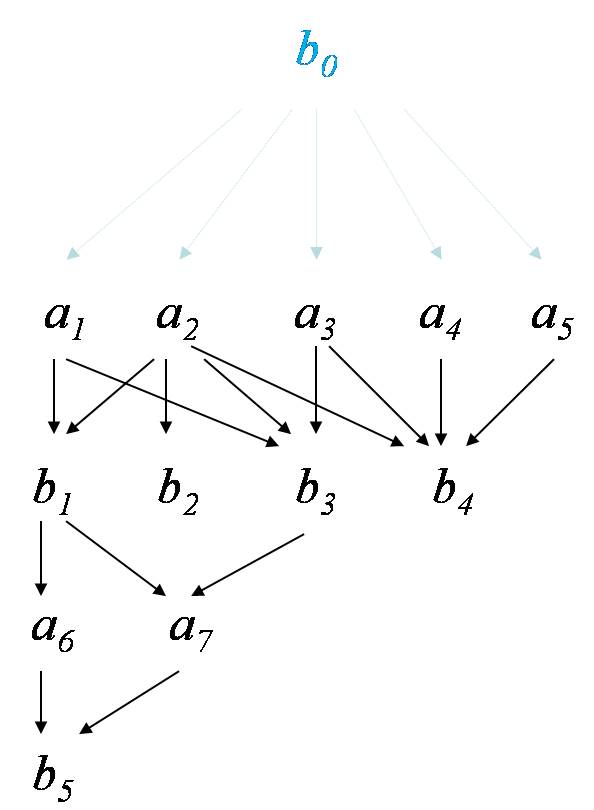}
\end{tabular}
\caption{The connected subgraphs of $\mathcal{G}_A$ in Figure~\ref{fig:co2} obtained after $a_2$, $a_3$, and $b_0$ are removed, respectively, from $\mathcal{G}_A$} \label{fig:subgraphs}
\end{figure}

\section{Further Discussion and Related Works}
In this section, we provide a detailed comparison of our work with related works. Section~6.1 discusses related works on various kinds of complex regions. Section~6.2 then compares our model with the tree model of \cite{WorboysB93}, and shows how to derive the tree model from the link graph model.  We show in Section~6.3 how our graph model can be used to produce representations of bounded regions at multiple levels of detail.

\subsection{Related Works on Complex Regions}
The notion of a simple region with holes and its generalized region (called `envelope' in this paper) were first introduced in \cite{EgenhoferCF94}. Our definition is slightly different from that given in \citep{EgenhoferCF94}, but is consistent with the one used in \citep{SchneiderB06}.

In  \citep{EgenhoferCF94} Egenhofer et al. proposed  two definitions of simple regions with holes, both requiring that the region has a connected interior. The first definition further requires that the boundaries of different exterior components should be disjoint. It is clear that this restriction is too strong in many real-world applications. While this constraint is relaxed in the second definition, it allows `spikes' in the exterior of such a region. As a consequence, such a simple region with holes may not be a regular closed set, hence not a plane region in our sense. Figure~\ref{fig:2srhs} shows such an example, where the `region' $a$ in the left of this figure has a spike (\emph{viz.} the right edge of $h$) in the exterior of $a$, which makes it a non-regular closed set because  $a^\circ$ (showing in the right of this figure) is one-piece and its closure $\overline{a^\circ}$ is a simple region strictly contained in $a$.

\begin{figure}
\centering
\begin{tabular}{c}
\includegraphics[width=.5\textwidth]{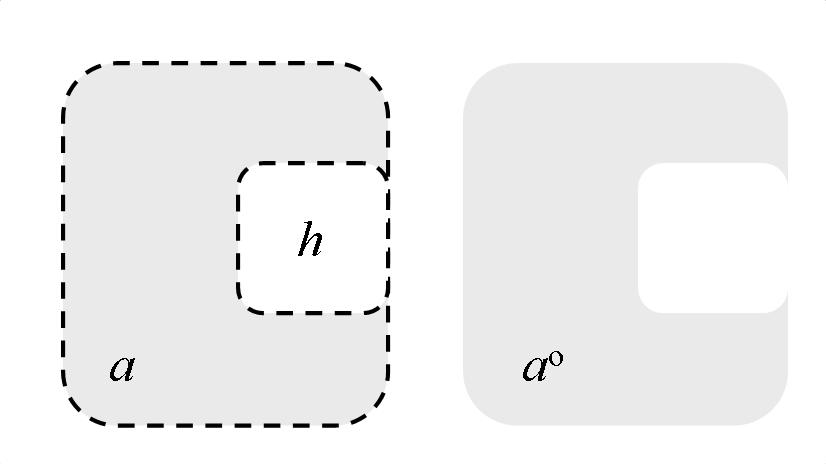}
\end{tabular}
\caption{A simple region with holes in the sense of \citep{EgenhoferCF94} which is not regular closed} \label{fig:2srhs}
\end{figure}

The notion of a composite region was first introduced in  \cite{ClementiniDC95}, where two components are allowed to meet at more than one points. This implies that a composite region in the sense of \cite{ClementiniDC95} may have a disconnected exterior, which is not allowed in our definition. For example, the region in Figure~\ref{fig:3comb}(c) is considered to be a composite region in \cite{ClementiniDC95}, but not a composite region in our sense.

Our internal structure model provides a unified description of various bounded regions. In particular, our Proposition~\ref{prop:envelope-is-cr} shows that, for a bounded semi-algebraic region $A$, the faces and holes of $A$ are all simple regions with holes, while the envelope of $A$ is a composite region. The link graph of $A$ shows that $A$ can be constructed iteratively by using the notions of composite region and simple region with holes. Roughly speaking, we have the following slogan
\begin{equation*}
\mbox{Complex Regions $=$ Simple Regions with Holes $+$ Composite Regions.}
\end{equation*}

Schneider and Behr \cite{SchneiderB06} defined a complex region as a bounded region $A$ which has finite \emph{faces} and the intersection of any two faces is a finite set. They considered a complex region to be a collection of faces. For example, the bounded region shown in Figure~\ref{fig:co2} is interpreted as having seven faces: one has a hole, the other six are simple regions. How these faces are composed is not mentioned explicitly in their definition. As another example, the two complex regions in Figure~\ref{fig:SBcomplex} have the same set of faces, but different link graphs and internal topological structures.
When restricted to semi-algebraic sets, all bounded regions considered in this paper are complex regions in the sense of \cite{SchneiderB06}.

\begin{figure}
\centering
\begin{tabular}{cccc}
\includegraphics[width=.3\textwidth]{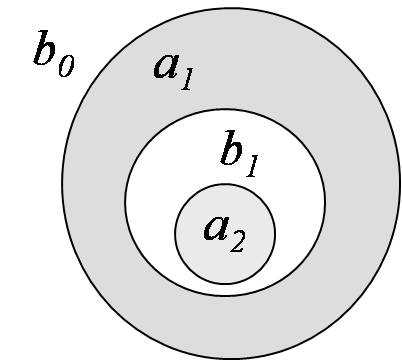}
&
\includegraphics[width=.06\textwidth]{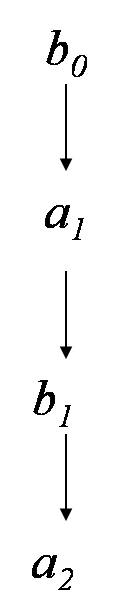}
&
\includegraphics[width=.35\textwidth]{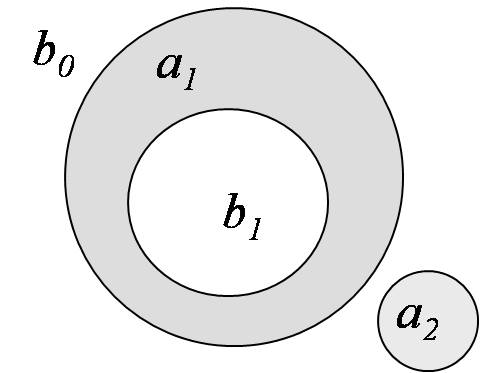}
&
\includegraphics[width=.18\textwidth]{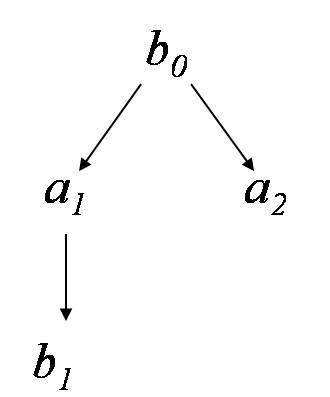}
\end{tabular}
\caption{Two complex regions and their link graphs.} \label{fig:SBcomplex}
\end{figure}

While different types of plane regions have been introduced, the above works take a complex region as a whole and focus on the topological relations between two complex regions. These works do not investigate the internal structure of general complex regions.

\subsection{The Worboys-Bofakos Tree Model}

The work  of Worboys and Bofakos  \cite{WorboysB93}  is perhaps the work most closely related to ours. Suppose $A$ is a bounded region that has a finite representation. In \citep{WorboysB93}, each region $A$ is represented as a tree (called the W-B tree of $A$), where the root node has no special meaning (could be interpreted as $A$ itself), and each non-root node represents a simple region. For two non-root nodes $m,n$, if $m$ is a child node of $n$, then $m$ is contained in $n$, and the intersection of $m$ and the boundary of $n$ is a finite set. All child nodes of $n$ form a composite region in our sense (called a \emph{base area} in \citep{WorboysB93}), which may separate $n$ into pieces. Furthermore, all nodes with depth $i$ also form a composite region, denoted by $A_i$, and called the base area of $A$ at level $i$. It is easy to see that $A_1=\widehat{A}$, and $ A_i\supseteq A_{i+1}$ for any $i\geq 1$.

Given a bounded semi-algebraic region $A$, we describe a method for computing the W-B tree of $A$. First, we next introduce a root node to denote $A$ itself; then compute all atoms of $\widehat{A}$ using Proposition~\ref{prop:envelope-is-cr}, and add these atoms as the children of the root node. For each atom $o$ of $\widehat{A}$, we consider the components of $A$ that are contained in $o$ but meet $\partial o$ at at most finite points. These components form another bounded semi-algebraic region, written $A_o$. We note that $A_o$ has a simpler internal structure than $A$, as it has fewer atoms and fewer levels. Suppose we have already computed the W-B tree of $A_o$. Putting the W-B tree of $A_o$ (with the root removed) directly under $o$, and repeating this procedure for all other atoms of $\widehat{A}$, we will obtain the W-B tree of $A$. 

Interestingly, the W-B tree of each bounded semi-algebraic region $A$ can also be computed directly from its link graph $\mathcal{G}_A$. Write $\mathcal{G}^{ i}_A$ for the restriction of $\mathcal{G}_A$ to components with level bigger than or equal to $i$, where $i$ is a natural number no bigger than the highest level of components of $A$. Recall that Proposition~\ref{prop:generegion} asserts that each atom of $\widehat{A}$ (\emph{i.e.} the base area of $A$ at level 1) is identical to the closure of the union of the bounded components in a unique maximally connected subgraph of $\mathcal{G}^{ 1}_A$. In general, it can similarly be shown that each atom of $A_i$ (the base area of $A$ at level $i$) is identical to the closure of the union of the bounded components in a unique maximally connected subgraph of $\mathcal{G}^{ i}_A$. Moreover, an atom $o$ of $A_{i+1}$ is a child of an atom $o'$ of $A_i$ iff the maximally connected subgraph of  $\mathcal{G}^{ i+1}_A$ corresponding to $o$ is contained in the maximally connected subgraph of $\mathcal{G}^{ i}_A$ corresponding to $o'$. In this way, we will obtain the W-B tree of $A$ from $\mathcal{G}_A$. Figure~\ref{fig:co2} shows a bounded region, its link graph, and its W-B tree.  It is easy to see that  only $\mathcal{G}^{ 2}_A$ has more than one maximally connected graph, and the corresponding atoms of the base area at level 2 are $\overline{b_2},\overline{b_4}$, and the closure of $b_1\cup b_3\cup b_5\cup a_6\cup a_7$.

The link graph of a bounded semi-algebraic region, however,  cannot be derived from its W-B tree.  Consider the two regions $A$ and $C$ illustrated in Figure~\ref{fig:3comb}. These two regions have non-isomorphic link graphs (cf. Figure~\ref{fig:graph-b}). Nevertheless, $A$ and $C$ have the same tree model, which is a chain with three nodes. In particular, the two faces $\overline{a_1},\overline{a_2}$ do not appear in the tree model of $C$. This also implies that the atoms appearing in the tree model are not sufficient for determining the global 9-intersection relation.

The reader may have noticed that this difference between the tree model and the link graph model mainly arises because the holes separate their carrier (\emph{i.e.} their parents in the link graph or their parent in the tree model). Suppose this never happens in a bounded region $A$, \emph{i.e.}, no node in the link graph has two or more parents. Then the link graph of $A$ is a tree and hence identical with the W-B tree of $A$.

While in the above sections we have justified, theoretically, the importance of our internal topological structure model, we next justify, informally, its applicability in practical spatial problems.
\begin{remark}\label{rmk:pseduhole}
While geographic entities that have multiple faces and/or holes are common, we seldom see that a geographic entity consists of multiple faces which are mutually connected at more than one point. Despite this, spatial objects stored in computers are often (polygonal or semi-algebraic) approximations of real-world geographic entities. It is very possible that two faces of a complex polygon may meet at two or more points. For instance, consider the water body in the south of East Europe formed by the Black Sea and the Sea of Azov.\footnote{See \url{http://upload.wikimedia.org/wikipedia/commons/5/52/Black_Sea_map.png}} This geographic entity has a hole, namely the Crimean Peninsula. At a coarse resolution, this water body has a representation whose internal topological structure is the same as region $C$ in Figure~\ref{fig:3comb}.

The model could also be used to describe topological changes in dynamic spatial phenomena \citep{JiangW09}. For example, consider a bushfire in a forest represented as a simple region. The fire starts at the center of the forest, and spreads and then separates the forest into two pieces. It is not hard to imagine that, at certain time point, the fire may possibly reach the two ends of the forest simultaneously. The scenario of this snapshot has the same internal topological structure as region $C$ in Figure~\ref{fig:3comb}.
\end{remark}
In the next subsection, we will discuss another application of our model in practical spatial problems.

\subsection{Generalization by Dropping}
Map generalization is a very important technique used in cartography and GISs.  \cite{EgenhoferCF94} introduced a generalization method for simple regions with holes. For a simple region with holes $A=(a_0;a_1,\cdots,a_k)$, we obtain
its envelope $a_0$ by dropping all its holes. In this section, we extend this idea to produce representations of bounded regions at multiple levels of detail.

Suppose $A$ is a bounded semi-algebraic region. We obtain its envelope by merging all its bounded components.  Combined with qualitative size information \citep{GereviniR02}, we can also
obtain less complicated regions by dropping \emph{some} small components,
together with their holes, from a bounded semi-algebraic region.

\begin{figure}
\centering
\begin{tabular}{cccc}
\includegraphics[width=.35\textwidth]{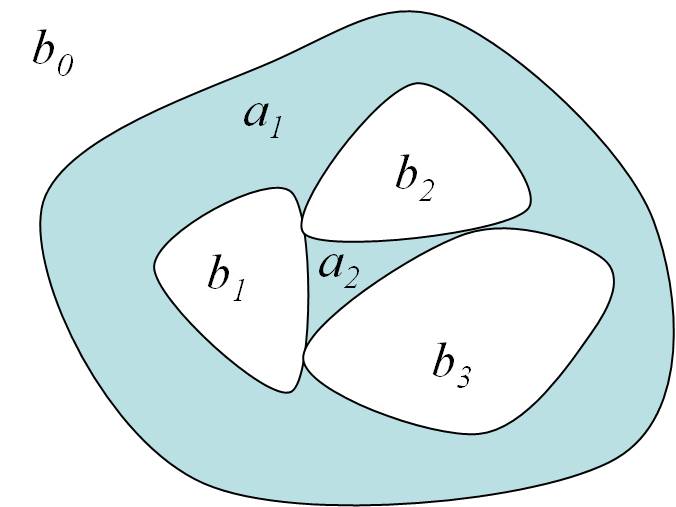}
&
\includegraphics[width=.15\textwidth]{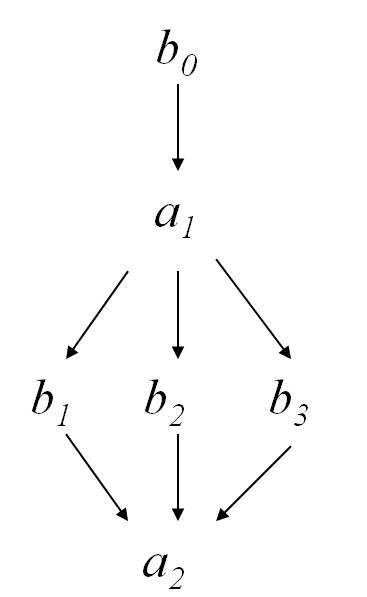}
&
\includegraphics[width=.35\textwidth]{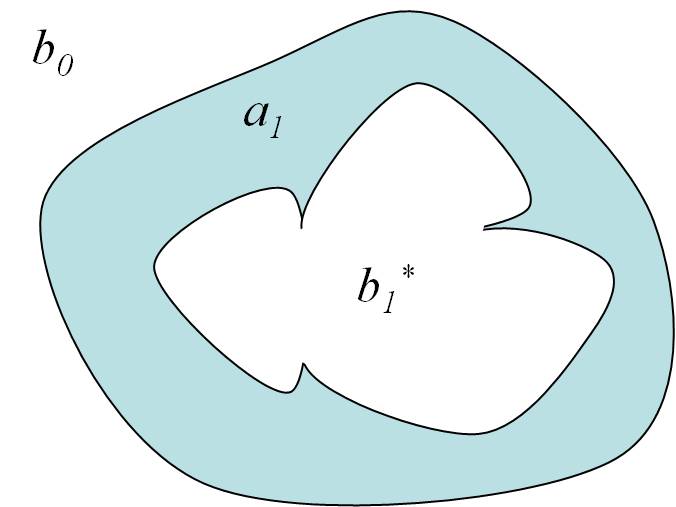}
&
\includegraphics[width=.05\textwidth]{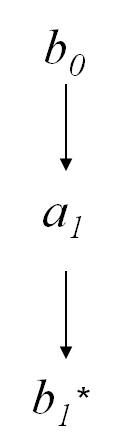}
\end{tabular}
\caption{A complex region (left) and its generalization (right),
where $b_1^\ast$ is the interior of the closure of $a_2\cup b_1\cup b_2\cup b_3$.} \label{fig:co3a}
\end{figure}

A bounded semi-algebraic region can be generalized step by step by dropping one
face/hole in each step. Let $A$ be a bounded semi-algebraic region. Suppose $c$ is a bounded component of $A$ and 
$\overline{c}$ has no children. Let $p_1,\cdots,p_k$ be the
parents of $c$ in the link graph. It is easy to see that
$T=\overline{c}\cup \bigcup_{i=1}^k \overline{p_i}$ is also a connected region. In
this way, we obtain a new region $A_1$ by merging the hole
$\overline{c}$ into its parents. Note that each bounded component $c^\prime\not\in \{c,p_1,\cdots,p_k\}$  of $A$ is a component of $A_1$. Moreover, $T$ is a face/hole of $A_1$. Write $r$ for the interior of $T$. The link graph of $A_1$ is obtained by setting
\begin{itemize}
\item $V_1=(V\setminus\{c,p_1,\cdots,p_k\})\cup\{r\}$;
\item $E_1=(E\cap V_1\times V_1)\cup
\{(m,r):(\exists i)(m,p_i)\in E\}.$
\end{itemize}
Figure~\ref{fig:co3a} shows a bounded semi-algebraic region and one step of its
generalization. Note that after each step of dropping, we obtain a bounded semi-algebraic region that has fewer faces/holes. Step by step, we will obtain a region that cannot be generalized. This is exactly $\widehat{A}$,
the \emph{envelope} of $A$.

\section{Conclusion}
In this paper we have established a qualitative model for
representing the internal topological structure of bounded plane regions. For
the first time, we have introduced the notion of holes, envelopes, and atoms to general bounded regions. Assuming $A$ is a bounded semi-algebraic region, we proved that the envelope of $A$ is a composite region and each face/hole of $A$ is a simple region with holes. In this way, we elegantly put the concepts of a composite region and a simple region with holes in the general framework of complex regions. 

The atoms of $A$ are naturally defined as the simple regions involved in the envelope and faces and holes of $A$. We have proved that these atoms are necessary and sufficient for
determining the nine-intersection relation between bounded semi-algebraic regions.
We believe this provides a partial justification for the rationale 
of applying the 9IM approach on complex regions. It also suggests
that the 9-intersection relation between complex regions can be implemented
through the implementation of the RCC8 (or the 9IM) relation between
simple regions.

For each bounded semi-algebraic region $A$, we also constructed a layered graph, called the link graph of $A$, to represent the internal topological structure of $A$. Given the link graph of $A$, we also described a method for computing the W-B tree of $A$. This shows that the link graph is finer than the tree model. Moreover, the link graph can be used to answer spatial queries that the tree model cannot answer. An example of such a query is  ``How many faces does $A$ (or a hole of $A$) have"? Furthermore, the layered graph representation also provides a natural way to generalize complex regions. This will play an important role in automatic cartography and GISs.

Future work will consider spatial reasoning with this internal structure model. Some related work has been carried out for simple regions with holes \citep{VasardaniE09}.

\appendix
\section{Proof of Proposition~\ref{prop:envelope-is-cr}}
In this section, we provide rigorous proof for Proposition~\ref{prop:envelope-is-cr}. We recall that a set in $\mathbb{R}^2$ is an open arc if it is homeomorphic to the open unit interval $(0,1)$. For an undirected graph $G=(N,E)$, a \emph{cycle} in $G$ is a sequence $(v_0,e_1,v_1,\cdots,e_d,v_d)$ so that $e_i$ is an edge connecting $v_{i-1}$ and $v_i$ ($i=1,\cdots,d$), with $v_0=v_d$; the \emph{degree} of a node $P$ in $N$ is the number of edges in $E$ emanating from $P$. 

Suppose $A$ is a bounded semi-algebraic region, and $c$ is an interior or exterior component of $A$. We define an undirected graph $\mathbf{G}_c=(\mathbf{N}_{c},\mathbf{E}_c)$ and show that all its nodes have even degrees.  It is then well-known that $\mathbf{G}_c$ can be decomposed into pairwise edge-disjoint cycles, \emph{i.e.}, no two cycles have an edge in common, and moreover, no cycle traverses the same edge more than once (see \emph{e.g.} \cite{Trudeau1994}). Furthermore, it is straightforward to show that $\mathbf{G}_c$ can be decomposed into pairwise edge-disjoint cycles that are simple, where a cycle is \emph{simple} if there are no repeated nodes other than the starting and ending node. Our proof of Proposition~\ref{prop:envelope-is-cr} is based upon the above observation.

We begin with the construction of $\mathbf{G}_c=(\mathbf{N}_{c},\mathbf{E}_c)$. By Proposition~\ref{prop:decomposition-sasets}, we suppose $\partial A$ is the disjoint union of $k$ points in $\mathbf{N}=\{P_1,\cdots,P_k\}$  and $l$ open arcs in $\mathbf{E}=\{\varphi_1,\cdots,\varphi_l\}$, where each $\varphi_j$ has two endpoints, both are in $\mathbf{N}$, and each point $P_i$ is an endpoint of some arc in $E$.  Regard points in $\mathbf{N}$ as nodes, and connect two nodes by an arc $\varphi$ if they are the two endpoints of $\varphi$. Let $\mathbf{G}=(\mathbf{N}, \mathbf{E})$ be the undirected graph with node set $\mathbf{N}$, and edge set $\mathbf{E}$. Note that two nodes may be connected by more than one arc. 

Take an arbitrary but fixed point $Q$ in the component $c$ of $A$. We say a point $P$ in $\partial A$ is \emph{reachable} (in $c$ from $Q$) if there is an arc $\alpha$ that connects $P$ to $Q$ and $\alpha\cap\partial A=\{P\}$. Because $c$ is a connected component, we know $\alpha\setminus\{P\}$ is contained in $c$. For any open arc $\varphi$ in $\mathbf{E}$, it is easy to see that a point in $\varphi$ is reachable iff all points in $\varphi$ are reachable. We therefore say an open arc is reachable if it contains a point that is reachable. Write $\mathbf{E}_c$ for the set of open arcs $\varphi$ in $\mathbf{E}$ that are reachable, and $\mathbf{N}_{c}$ for the set of points in $\mathbf{N}$ that are reachable. It is routine to check that the endpoints of each reachable arc are  reachable. Write $\mathbf{G}_c$ for the undirected graph $(\mathbf{N}_{c},\mathbf{E}_c)$.  

We observe that each node in $\mathbf{G}$ ($\mathbf{G}_c$) has an even degree. 
\begin{lemma}
\label{lemma:even-degree}
The degree of each node $P$ in $\mathbf{G}$ (or $\mathbf{G}_c$) is even.
\end{lemma} 
\begin{proof}
Write $\varphi_{j_1},\cdots,\varphi_{j_s}$ for all arcs in $E$ emanating from $P$ and suppose $\varphi_{j_t}$ is (clockwise) the next arc of $\varphi_{j_{t-1}}$ ($1< t\leq s$, and assume $\varphi_{j_{s+1}}=\varphi_{j_1}$). Take a sufficiently small open disk $U$ centred at $P$. Write $\triangle_t$ for the intersection of $U$ with the area between two consecutive arcs $\varphi_{j_t}$ and $\varphi_{j_{t+1}}$. Then $\triangle_t$ is either contained in $A^\circ$ or $A^e$. Moreover, $\triangle_{t}$ and $\triangle_{t+1}$ cannot be both contained in $A^e$ or both contained in $A^\circ$.  This implies that the degree of $P$ in $\mathbf{G}$ is even.

Similarly, for each reachable point $P$ in $\mathbf{N}$, we show that there are $2m$ $(m\geq 1)$ reachable arcs emanating from $P$. Because $P$ is reachable, there exists at least one $\triangle_t$ that is contained in $c$. For any $1\leq t<s$, we note that $\triangle_t$ and $\triangle_{t+1}$ cannot both be contained in $c$. This shows that there are $2m$ reachable arcs emanating from $P$, where $m$ is the number of $\triangle_t$ contained in $c$.
\end{proof}

We now show \emph{the envelope of $A$ is a composite region.}
\begin{proof}
Let $b_0$ be the unbounded exterior component of $A$. Suppose the graph $\mathbf{G}_{b_0}$ can be decomposed into pairwise edge-disjoint cycles $\gamma_1,\gamma_2,\cdots,\gamma_s$ that are simple. By abuse of notation, we also denote by $\gamma_i$ the subset of $\mathbb{R}^2$ that is the union of the nodes (\emph{i.e.} points) and edges (\emph{i.e.} open arcs) in $\gamma_i$. It is easy to see that each $\gamma_i$ is a Jordan curve. Write $d_i$ for the simple region with boundary $\gamma_i$. 

We first show that $\widehat{A}=\bigcup_{i=1}^s d_i$.  Recall that a point $P$ is reachable (in $b_0$ from a pre-chosen point $Q$) if there is an open arc $\alpha_P$ contained in $b_0$ which connects $Q$ to  $P$. By the choice of nodes and edges in $\mathbf{G}_{b_0}$, we know that each point in $\gamma_i=\partial d_i$ is reachable. Because $b_0$ is the unbounded exterior component of $A$, we assert that any point outside $\bigcup_{i=1}^s d_i$ is also reachable. Suppose $P'$ is a point outside $\bigcup_{i=1}^s d_i$. Take $P\in \bigcup_{i=1}^s d_i$ such that $d(P',P)=d(P',\bigcup_{i=1}^s d_i)$. It is easy to see that $P\in \bigcup_{i=1}^s \gamma_i$ and the path $\alpha_P\cup \overline{PP'}$ connects $Q$ to $P'$. This path could be modified as an open arc in $b_0$ that connects $Q$ to $P'$.  This means that the exterior of $\bigcup_{i=1}^s d_i$ is contained in $b_0$. On the other hand, by Jordan Curve Theorem, no interior point of $d_i$ ($1\leq i\leq s$) is reachable. This implies that $b_0$ is disjoint from $\bigcup_{i=1}^s d_i$. Therefore, we have $\bigcup_{i=1}^s d_i$ is the set complement of $b_0$, which happens to be the envelope of $A$. Hence $\widehat{A}=\bigcup_{i=1}^s d_i$. 

We next show that $d_1,d_2,\cdots,d_s$ are the faces of $\widehat{A}$. Because a reachable point is not in the interior of any $d_i$, we have $d_i^\circ\cap \partial d_j=\varnothing$ for any $i\not=j$.  We show this implies that $d_i^\circ\cap d_j^\circ=\varnothing$  for any $i\not=j$. Suppose this is not the case. Take a point $P\in d_i^\circ\cap d_j^\circ$. Note that because $d_i^\circ\cap \partial d_j=\varnothing$ we have either $d_i^\circ\not\subseteq d_j$ or $d_j^\circ\not\subseteq d_i$. Take $d_i^\circ\not\subseteq d_j$ as an example. We have a point $Q_1\in d_i^\circ$ but $Q_1\not\in d_j$. By $P\in d_j^\circ$, but $Q_1\not\in d_j$, we know any arc that connects $P$ to $Q_1$ should intersect the boundary of $d_j$. Because $P$ and $Q_1$ are both in the connected open set $d_i^\circ$, there is an arc $\varphi$ contained in $d_i^\circ$ that connects $P$ to $Q_1$. So we have $\varphi \cap \partial d_j \not=\varnothing$ and $\varphi\cap \partial d_j \subseteq d_i^\circ$. This contradicts the assumption that $ d_i^\circ \cap \partial d_j =\varnothing$.

Because $\gamma_i$ and $\gamma_j$ have no common edges, we know $\gamma_i\cap \gamma_j$ is a finite set. 
If $P_1,P_2$ are two different points in $\gamma_i\cap \gamma_j$, since $\gamma_i$ is a simple curve, we have two arcs  $\alpha_1,\alpha_2$ contained in $\gamma_i$ that connect $P_1$ to $P_2$. Similarly, there are two arcs $\beta_1,\beta_2$ contained in $\gamma_j$ that connect $P_1$ to $P_2$. Without lack of generality, suppose $\beta_1\cap \gamma_i=\{P_1,P_2\}$, \emph{i.e.} $\beta_1$ has no other common points with  $\gamma_i$. This means that $\alpha_1\cup \beta_1$  and  $\alpha_2 \cup \beta_1$ are Jordan curves. Consider the simple regions bounded by $\alpha_1\cup \beta_1$  and, respectively, $\alpha_2 \cup \beta_1$. It is clear that one must contain the other, and the larger one also contains a boundary point of $\gamma_i$ in its interior. This contradicts the assumption that each boundary point in $\gamma_i$ is reachable.  Therefore, the intersection of $\gamma_i$ and $\gamma_j$ is either empty or a singleton. Because $\gamma_i=\partial d_i$, $\gamma_j=\partial d_j$, and $d_i^\circ\cap d_j^\circ=\varnothing$, we know $d_i\cap d_j$ is either empty or a singleton. As $\widehat{A}=\bigcup_{i=1}^s d_i$ has been proved, we know that $\widehat{A}$ is a composite region.
\end{proof}





Similarly, we can prove that \emph{each face or hole of $A$ is a simple region with holes.}
\begin{proof}
The proof is similar to that for the envelope. Let $c$ be a bounded (interior or exterior) component of $A$. Suppose the graph $\mathbf{G}_c$ can be decomposed into pairwise edge-disjoint cycles $\gamma_0,\gamma_1,\cdots,\gamma_s$ that are simple. By abuse of notation, we also denote by $\gamma_i$ the subset of $\mathbb{R}^2$ that is the union of nodes (\emph{i.e.} points) and edges (\emph{i.e.} open arcs) in $\gamma_i$. It is easy to see that each $\gamma_i$ is a Jordan curve. Write $d_i$ for the simple region with boundary $\gamma_i$. 

We first show, for any $1\leq i\not=j\leq s$, that $\gamma_i\cap \gamma_j$ is either empty or a singleton. Because $\gamma_i$ and $\gamma_j$ have no common edges, we know $\gamma_i\cap \gamma_j$ is a finite set. 
If $P_1,P_2$ are two different points in $\gamma_i\cap \gamma_j$, since $\gamma_i$ is a simple curve, we have two arcs  $\alpha_1,\alpha_2$ contained in $\gamma_i$ that connect $P_1$ to $P_2$. Similarly, there are two arcs $\beta_1,\beta_2$ contained in $\gamma_j$ that connect $P_1$ to $P_2$. Without lack of generality, suppose $\beta_1\cap \gamma_i=\{P_1,P_2\}$. Then $\alpha_1\cup \beta_1$ and $\alpha_2\cup \beta_1$ are Jordan curves. Consider the simple regions bounded by these Jordan curves. It is clear that one must contain the other, and the larger one also contains a boundary point of $\gamma_i$ in its interior. This contradicts the assumption that each boundary point in $\gamma_i$ is reachable.  Therefore, the intersection of $\gamma_i$ and $\gamma_j$ is either empty or a singleton.

Recall that $Q$ is taken from the bounded connected component $c$ of $A^\circ\cup A^e$. We know $Q$ is contained in the interior of some simple region $d_i$, which has boundary $\gamma_i$. Suppose this is not the case. Then $Q$ will be connected to another point outside $A$ by an arc which is outside all $d_i$. Therefore, another boundary point of $A$ that is not on any $\gamma_i$ should exist. This is a contradiction. 
Without lack of generality, suppose this simple region is $d_0$. We then show $\partial d_i\subseteq d_0$ for any other $d_i$. Suppose this is not the case. Then there exists $P\in\partial d_i$ which is not in $d_0$. Because $P$ is reachable from $Q$, which is in the interior of $d_0$, this leads to a contradiction. So we have $\partial d_i\subseteq d_0$, hence $d_i\subseteq d_0$ for any $1\leq i\leq s$.  It is clear that $Q$ is not in any $d_i$ for $1\leq i\leq s$. Because $c$ is a connected component of   $A^\circ\cup A^e$ and $c$ contains $Q$, we know $c$ is contained in $A^\circ\setminus \bigcup_{i=1}^s d_i$. On the other hand, suppose $P\in A^\circ\setminus \bigcup_{i=1}^s d_i$. Let $P'$ be a point in $\bigcup_{i=0}^s{d_i}$ such that $d(P,P')=d(P,\bigcup_{i=0}^s{d_i})$. Then $P'\in\bigcup_{i=1}^s \gamma_i$. Suppose $\alpha_P$ is an open arc contained in $c$ which connects $Q$ to  $P$. It is clear that $\alpha_P\cup \overline{PP'}$ connects $Q$ to $P'$. Although this is not an open arc in $c$, it could be easily modified into an open arc in $c$ that connects $Q$ to $P'$. This means that any point $P$ in $A^\circ\setminus \bigcup_{i=1}^s d_i$ is also in $c$. Therefore, $c=A^\circ\setminus \bigcup_{i=1}^s d_i$. As a consequence, we know $\overline{c}=(d_0;d_1,\cdots,d_s)$ is a simple region with holes.
\end{proof}

\bibliographystyle{plain}

\bibliography{compob}
\end{document}